\documentclass[]{ROLA_tech_report}

\usepackage{dblfloatfix}
\usepackage{float}
\usepackage{algorithm}
\usepackage{algpseudocode}
\usepackage{hhline}
\usepackage{wrapfig}
\usepackage{url}
\definecolor{ourshade}{RGB}{230,230,250}
\newcommand{\best}[1]{\textbf{#1}}

\newcommand{\cmark}{\textcolor{green!55!black}{$\checkmark$}}
\newcommand{\xmark}{\textcolor{red!70!black}{$\times$}}
\newcommand{\pmark}{\textcolor{orange!85!black}{$\approx$}}
\definecolor{tridentRed}{HTML}{B9472F}
\definecolor{tridentOrange}{HTML}{E97845}
\definecolor{tridentLight}{HTML}{FFF0E9}
\definecolor{tridentHeader}{HTML}{FFF7F2}
\definecolor{tridentGray}{HTML}{666666}

\newcommand{\R}{\mathbb{R}}
\newcommand{\bR}{\mathbf{R}}

\newcommand{\softmax}{\operatorname{softmax}}
\newcommand{\SiLU}{\operatorname{SiLU}}
\newcommand{\sigmoidf}{\sigma}
\newcommand{\RMSNorm}{\operatorname{RMSNorm}}
\newcommand{\dd}{d_h}

\newcommand{\Pq}{P_q}
\newcommand{\Pk}{P_k}
\newcommand{\qlr}{\tilde{q}}
\newcommand{\klr}{\tilde{k}}
\newcommand{\bigo}{\mathcal{O}}

\newcommand{\rr}{r}

\title{RoLA: Rotary-Positioned Low-Rank Linear Attention\\ for Efficient Diffusion Transformers}

\author[1]{Zekun Zhang$^*$}
\author[1]{Yixiang Cai$^*$}
\author[14]{Yuxi Liu$^{*\ddagger}$}
\author[4]{Tengxu Sun}
\author[3]{Tianle Liu}
\author[1]{Zhoutong Wu}
\author[24]{Haoyu Li}
\author[4]{Baole Ai}
\author[4]{Ang Wang}
\author[4]{Jiamang Wang}
\author[4]{Lin Qu}
\author[1]{Kun Yuan$^\dagger$}
\affiliation[1]{Peking University, Melon Group}
\affiliation[2]{Tsinghua University}
\affiliation[3]{University of Electronic Science and Technology of China}
\affiliation[4]{Alibaba Group}
\firstpagenotes{$^*$ Equal contribution.\\ $^\dagger$ Corresponding author.\\ $^\ddagger$ Project leader.}
\appto{\affiliationlist}{%
\par\vskip 1.5mm
{\affiliationfont\sffamily\bfseries
GitHub:
\href{https://github.com/AlibabaResearch/SparkDiffusion}{%
{\rmfamily\bfseries https://github.com/AlibabaResearch/SparkDiffusion}%
}%
}%
}
\paperemail{%
\href{mailto:zhangzekun@std.uestc.edu.cn}{zhangzekun@std.uestc.edu.cn}
\quad
\href{mailto:caiyixiang@bupt.edu.cn}{caiyixiang@bupt.edu.cn}
\quad
\href{mailto:yuxiliu666@stu.pku.edu.cn}{yuxiliu666@stu.pku.edu.cn}
\quad
\href{mailto:kunyuan@pku.edu.cn}
{kunyuan@pku.edu.cn}
}
\date{\today}

\abstract{
Diffusion Transformers (DiTs) achieve strong video generation quality, but their dense spatiotemporal
self-attention scales quadratically with sequence length and quickly becomes the dominant inference
bottleneck. Sparse low-rank hybrids alleviate this cost by combining a local sparse branch with a
global compressed branch. In video DiTs equipped with 3D Rotary Position Embeddings (RoPE), the global
branch faces a structural compatibility issue: when RoPE is applied before a nonlinear feature map, the
rotation and nonlinearity generally do not commute, making it difficult to keep a query-independent
linear summary while preserving relative rotary geometry. Existing work often sidesteps this issue by
replacing genuine cross-token global aggregation with coordinate-conditioned surrogates or learnable
absolute positional modules. These compromises can be effective, but they approximate relative decay
from absolute coordinates and introduce extra positional parameters. We propose \textbf{RoLA}, a rotary-positioned
low-rank linear-attention branch that keeps genuine cross-token aggregation while remaining compatible
with a reusable linear summary. The design applies RoPE \emph{outside} the nonlinear low-rank feature
map and reuses a truncated subset of the pre-trained rotary schedule matched to the
low-rank bottleneck.
 This yields a linear-time low-rank global branch with relative positional behavior by
design and no additional positional parameters; the full sparse--low-rank module still includes the
fixed-sparsity sparse branch. Experiments on open-source video DiTs show that the resulting method
remains competitive in generation quality at 90\% sparsity while achieving 2.63$\times$ end-to-end
inference speedup on Wan2.1-14B (720p, 81 frames, measured on an NVIDIA H100 GPU).
}

\begin{document}

\maketitle

\begin{center}
\includegraphics[width=\linewidth]{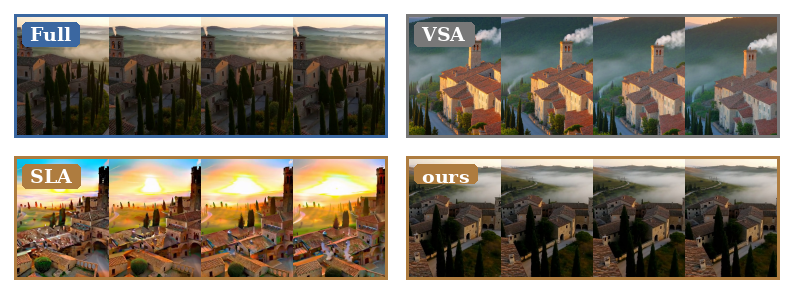}

\small Visual comparison of video generation quality across different attention methods.
\end{center}

\section{Introduction}
\label{sec:intro}

Diffusion Transformers (DiTs)~\citep{peebles2023dit} have become a dominant backbone for high-fidelity
video generation \citep{wan2025,kong2024hunyuan,yang2025cogvideox,hacohen2024ltx}. However, their dense
spatiotemporal self-attention scales quadratically with sequence length, making inference increasingly
expensive as video resolution and duration grow. Sparse attention
methods~\citep{xi2025svg,yang2025svg2,zhang2026faster,wu2025vmoba} reduce this cost by computing only a
selected subset of query--key interactions. While effective at moderate sparsity, aggressive sparsification often removes the
global interactions required for semantic consistency, motion coherence, and stable scene layout.

Sparse--linear hybrids address this limitation by augmenting a sparse local branch with a cheap global
branch~\citep{zhang2025sla,zhang2026sla2,fang2026salad}. In video DiTs, however, this design can encounter
a structural compatibility issue. On the one hand, 3D Rotary Position Embeddings (RoPE) encode relative
geometry through position-dependent rotations. On the other hand, linear attention relies on a
query-independent global summary that can be reused across tokens. If RoPE is pushed through a nonlinear
feature map in the usual way, the rotation and nonlinearity generally do not commute, so the resulting
summary can become entangled with absolute positions. We refer to this tension as the \textbf{RoPE
Dilemma}.

One practical workaround is to replace genuine cross-token aggregation in the global branch with
coordinate-conditioned surrogates and learnable absolute positional modules~\citep{liu2026ropeslr}. This
works, but it leaves two limitations. First, relative decay is approximated from absolute
coordinates rather than represented by the global branch itself. Second, the design gives up the
reusable linear-attention view of global context and introduces extra positional parameters.

This naturally raises the question: \emph{can we build a genuine linear global branch that remains compatible
with 3D relative-position structure?} We answer this question by rotating the already-compressed
low-rank features \emph{after} the nonlinearity. As a result, the branch keeps a reusable global summary
while its pairwise interactions remain functions of relative offsets. We further reuse a truncated
subset of the pre-trained rotary schedule that matches the low-rank bottleneck, avoiding any new
positional module. Together, these two choices turn the compressed branch from an absolute-coordinate
surrogate into a relative-position-aware linear branch, while preserving linear-time reuse and
a hardware-friendly implementation. Although this outside-activation ordering is simple in isolation,
making it work in a pre-trained 3D-RoPE video DiT requires resolving three coupled constraints: the
limited rotary frequency budget of the low-rank bottleneck, stable fusion with a sparse branch on the
backbone scale, and adaptation without disrupting pre-trained attention statistics. We refer to the
resulting method as \textbf{RoLA}.
We do not claim a new generic linear relative-position encoder; rather, RoLA is a video-DiT-specific
sparse-global construction that makes post-activation RoPE compatible with a reusable low-rank global
branch and a sparse deployment engine.
Figure~\ref{fig:efficiency} summarizes the target operating regime: compared with dense attention and
representative sparse or sparse--global baselines, the proposed method reduces attention cost and
end-to-end latency while retaining a genuine global branch. These deployment-oriented results use a
separate RTX 5090 long-sequence profiling setup; they illustrate scaling behavior rather than provide a
direct latency counterpart to the H100-based main benchmark table.

\noindent\textbf{Contributions.} Our contributions are as follows:
\begin{itemize}[leftmargin=1.4em,itemsep=1pt,topsep=2pt]
\item \textbf{Structural compatibility issue.} In 3D-RoPE video DiTs, a genuine low-rank global branch
must preserve relative-position geometry while remaining reusable as a query-independent linear
summary. These two requirements conflict under the usual nonlinear linear-attention formulation.
\item \textbf{Compatible rotary low-rank branch.} The proposed branch applies RoPE after the nonlinear
low-rank feature map and reuses a rank-truncated subset of the pre-trained rotary schedule.
This preserves relative-position structure at the interaction level while keeping a reusable
linear-time global summary, without introducing new positional parameters.

\item \textbf{Mechanistic and empirical validation.} Mechanism-level experiments on real features from
mainstream open-source video models verify the branch's relative-position behavior, and generation
experiments across representative backbones demonstrate favorable quality--efficiency trade-offs under
high sparsity and full-stack inference.
\end{itemize}

\begin{figure*}[!b]
\centering
\includegraphics[width=\textwidth]{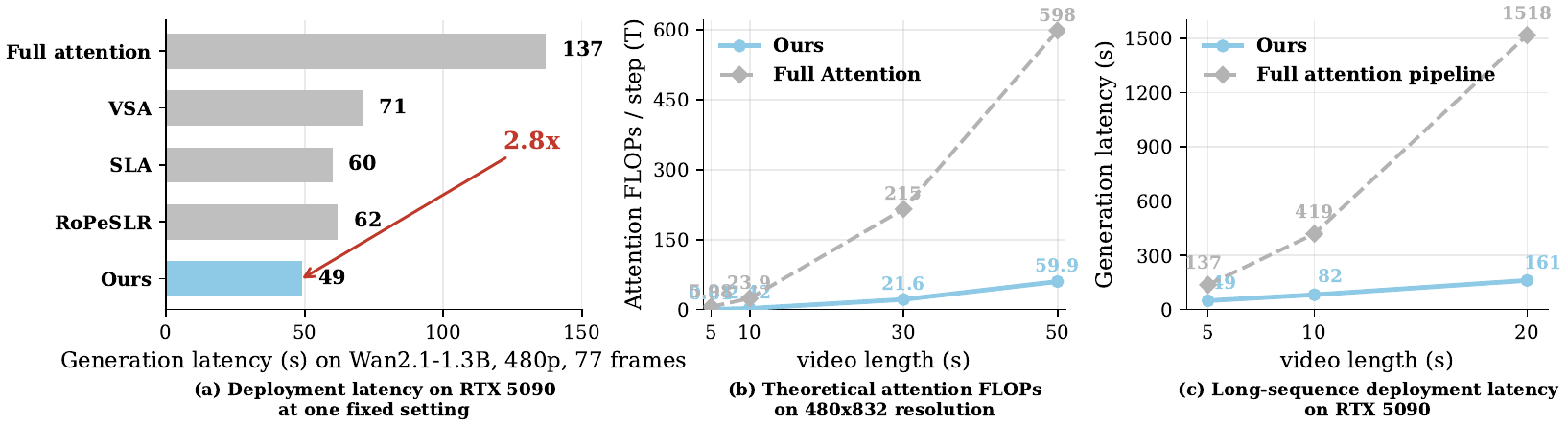}
\caption{Deployment-oriented efficiency summary on RTX 5090 (different from the benchmark evaluation
setting). Panel (a) shows representative fixed-setting deployment latency, panel (b) shows theoretical
self-attention FLOPs per denoising step, and panel (c) shows long-sequence deployment latency.}
\label{fig:efficiency}
\end{figure*}

\section{Related Work}
\label{sec:related}

\noindent\textbf{Sparse attention for video DiTs.} Sparse attention is a dominant approach to the
quadratic attention bottleneck in video DiTs, exploiting structured, learned, or dynamic sparsity in
spatiotemporal interactions. Sparse VideoGen \citep{xi2025svg} exploits spatial--temporal sparsity, while SVG2
\citep{yang2025svg2} further improves sparse token selection through semantic-aware permutation; Efficient-vDiT
\citep{ding2025efficient} leverages repetitive attention tiles, while Sparse-vDiT
\citep{chen2026sparse} identifies recurring structured sparsity patterns across heads and layers. VSA
\citep{zhang2026faster} learns trainable sparse patterns, VMoBA \citep{wu2025vmoba} organizes attention into a
mixture of blocks, Radial Attention \citep{li2026radial} exploits spatiotemporal energy decay for
distance-aware sparse computation, and MOD-DiT \citep{liu2026mixture} dynamically predicts sparsity
patterns across denoising intervals. Despite substantial efficiency gains, these sparse-only approaches
retain only a subset of query-key interactions; under aggressive sparsity, weak but collectively important
long-range context can be discarded, motivating complementary global modeling.

\noindent\textbf{Sparse linear hybrids.} To recover global context discarded by sparsification, recent
methods combine sparse attention with efficient linear branches. SLA \citep{zhang2025sla}, SLA2
\citep{zhang2026sla2}, SALAD \citep{fang2026salad}, and SVG-EAR \citep{zhou2026svg} explore trainable or
lightweight sparse--linear compensation for discarded global interactions. Conventional sparse--linear
formulations can nevertheless encounter the RoPE Dilemma (\cref{sec:prelim}): nonlinear feature maps
generally do not commute with rotary transformations, making it difficult to preserve 3D-RoPE
relative-position structure in a reusable linear branch. The limitation becomes more pronounced at extreme
sparsity, where the linear branch must recover more discarded interactions. Related linear video DiTs,
including SANA-Video \citep{chen2026sana} and SANA-Video 2.0 \citep{chen2026sana2}, further introduce
RoPE-aware and hybrid linear-attention designs. We treat these SANA-style models as a separate dedicated linear/hybrid-linear video-DiT line rather
than as direct baselines, because our main comparison targets sparse
post-training hybrids under the same video-DiT adaptation setting. RoPeSLR, a low-rank Fourier
compensator~\citep{liu2026ropeslr}, instead sidesteps the dilemma in sparse--low-rank attention by replacing
cross-token linear aggregation with a per-token coordinate MLP and a learnable absolute 3D positional
embedding. Building on this
sparse--global decomposition, our method retains genuine cross-token linear aggregation while directly
reusing the pre-trained 3D-RoPE structure in the low-rank branch.

\noindent\textbf{Rotary position embeddings.} RoPE \citep{su2021roformer} encodes relative position through
orthogonal rotations with an exponentially decaying frequency schedule. The compatibility of relative
positional encoding with linear attention has also been studied in LRPE \citep{qin2023linearized}, which
preserves relative positional modeling under linear-complexity computation. In video DiTs, however, the
head dimension is split across the three spatiotemporal axes (3D RoPE), so the per-axis frequency budget is
small; we reuse the first $\rr$ rotary coordinates in the backbone's native ordering, i.e., $\rr/2$ complete rotary coordinate pairs supported by the low-rank bottleneck, provide an idealized error-analysis intuition for this reuse (\cref{app:truncation-bound}), and validate the
resulting operator empirically (\cref{sec:trunc}; Fig.~\ref{fig:truncation}). Unlike LRPE, RoLA is not a generic standalone positional encoder; it is a sparse-video-DiT construction
that preserves a reusable cross-token global branch while remaining compatible with a pre-trained 3D-RoPE
backbone and the sparse engine.

\begin{figure}[t]
\centering
\includegraphics[width=\linewidth]{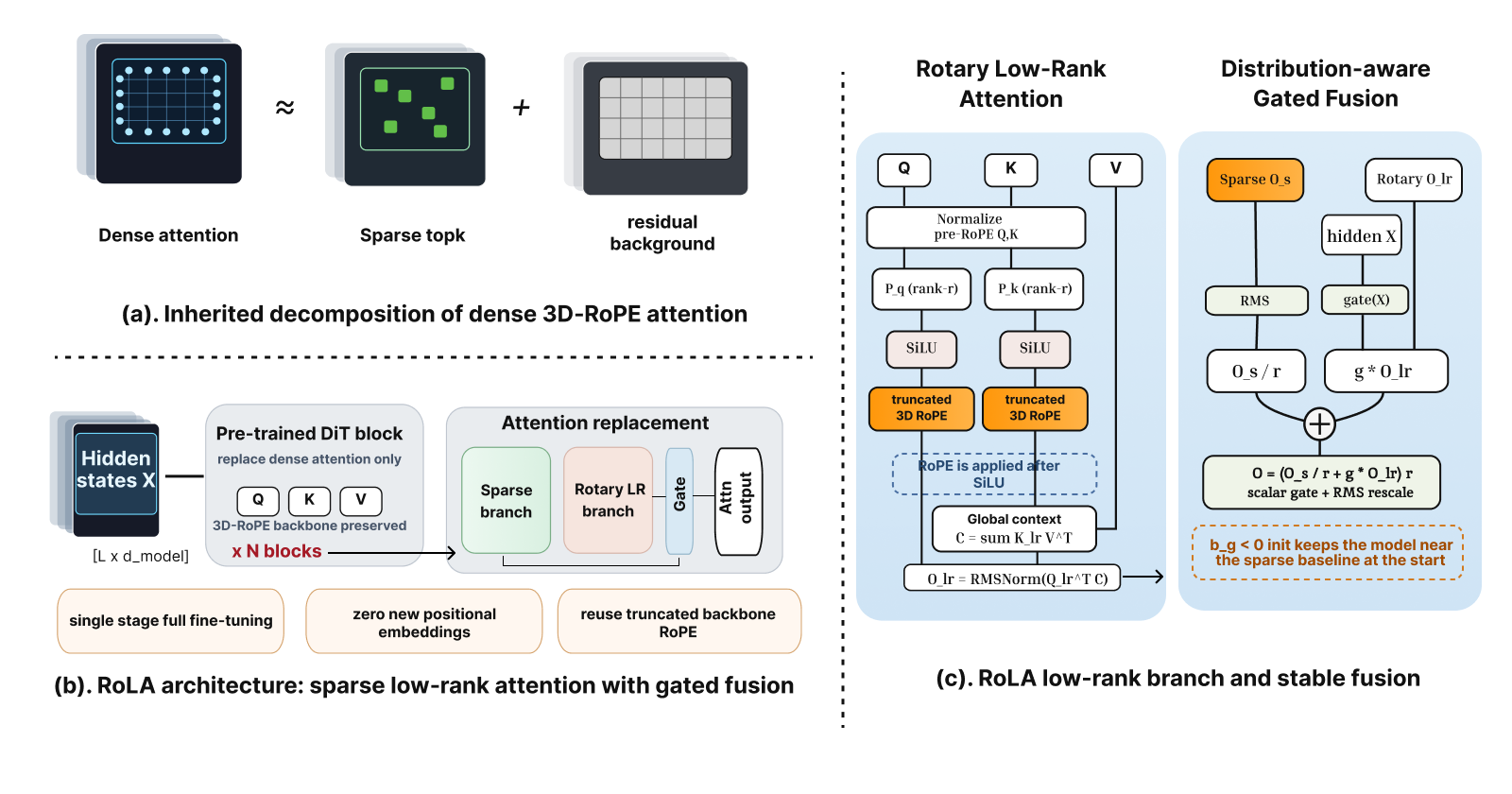}
\caption{Architecture overview of the proposed method. (a) Dense 3D-RoPE attention decomposes into sparse top-$k$
spikes and a residual background. (b) The method replaces dense attention in pre-trained DiT blocks with
sparse--low-rank branches and gated fusion, reusing backbone RoPE without additional positional embeddings.
The right part shows the rotary low-rank branch and distribution-aware gated fusion.}
\label{fig:framework}
\end{figure}

\section{Preliminary}
\label{sec:prelim}

\noindent\textbf{Notation.}
A video is flattened into $L$ spatiotemporal tokens indexed by
$p=(t,x,y)\in[T]\times[H_{\mathrm{s}}]\times[W_{\mathrm{s}}]$, with
$L=T H_{\mathrm{s}} W_{\mathrm{s}}$. Each head projects the input to
$Q,K,V\in\R^{L\times \dd}$, with per-token vectors $q_p,k_p,v_p\in\R^{\dd}$.
We use $p$ and $q$ as token-position indices; $q_p$ and $k_q$ denote the query and key vectors at those
positions. Full attention computes
$$
O=\softmax(QK^\top/\sqrt{\dd})V,
$$
whose $\bigo(L^2)$ cost is the central bottleneck for ultra-long sequences ($L\gtrsim10^4$) induced by
high-resolution video.

\noindent\textbf{3D RoPE.}
3D RoPE partitions the head dimension across the temporal and spatial axes, with
$\dd=d_t+d_x+d_y$, and applies a block-diagonal orthogonal rotation
$\bR(p)\in\R^{\dd\times \dd}$. Its orthogonality gives the relative-position identity
\begin{equation}
s_{p,q}
\;=\;
\tfrac{1}{\sqrt{\dd}}\langle \bR(p)q_p,\, \bR(q)k_q\rangle
\;=\;
\tfrac{1}{\sqrt{\dd}}\langle q_p,\, \bR(q-p)k_q\rangle .
\label{eq:shiftinv}
\end{equation}

\noindent\textbf{RoPE Dilemma.}
Linear attention \citep{katharopoulos2020transformers} relies on a query-independent global summary that
can be reused across queries, whereas RoPE encodes relative geometry through position-dependent rotations.
For a feature map $\phi$, an inside-RoPE linear branch would form a key-side summary such as
$\sum_q \phi(\bR(q)k_q)v_q^\top$. When $\phi$ is a pointwise nonlinearity, the rotation and the
nonlinearity generally do not commute, so the summary can depend on absolute positions in a way that cannot
be cleanly absorbed into a relative offset. This makes it difficult to simultaneously preserve relative
rotary geometry and maintain a reusable linear-time summary. This mismatch motivates RoLA: in the low-rank branch, RoPE is applied
after the low-rank nonlinearity so that position information is injected without destroying the reusable
global summary. Appendix~\ref{app:proofs} provides simple structural sanity checks for this design choice; we do not claim them as theoretical guarantees.

\begin{figure}[t]
\centering
\includegraphics[width=\linewidth]{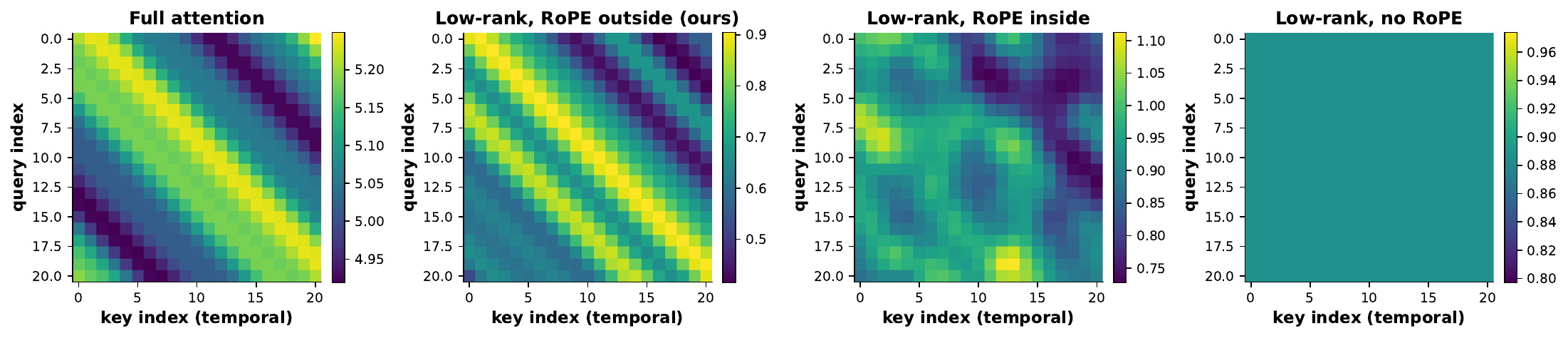}
\caption{Content-controlled positional kernels on a temporal slice. From left to right: full attention, the
low-rank branch with RoPE applied \emph{outside} the activation, the low-rank branch with RoPE applied
\emph{inside} the activation, and the low-rank branch without RoPE. Only the outside-activation design
closely matches the banded relative-distance pattern of full attention. The inside-activation variant
remains structured but distorts the relative geometry, while removing RoPE collapses the map to a nearly
flat, position-blind response. Color bars are shown per panel to visualize structure; the comparison focuses
on the spatial pattern rather than absolute color ranges.}
\label{fig:exp1}
\end{figure}

\section{Method}
\label{sec:method}

The proposed method decomposes attention into two complementary branches: a hardware-efficient block-sparse branch
that captures high-energy semantic spikes, and a low-rank linear branch that reconstructs the smooth global
background continuum. The sparse branch follows standard block-sparse attention, e.g., SLA/VMoBA
\citep{zhang2025sla,wu2025vmoba}, and is orthogonal to our main contribution; we therefore focus on the
low-rank branch. In experiments, we keep this sparse branch fixed across compared sparse-based methods
unless otherwise stated.

\noindent\textbf{Why keep a sparse-plus-global decomposition?}
Several efficient-attention formulations suggest viewing 3D-RoPE video attention as sharp sparse spikes
plus a smoother compressible background~\citep{liu2026ropeslr,zhang2025sla}. We therefore focus
on a narrower question: \emph{what should the compressed global branch look like if it must remain
linear-time while preserving compatibility with relative rotary geometry?}

We therefore do not claim the sparse-plus-low-rank decomposition itself as our contribution. This
decomposition is shared with prior hybrid designs. Our contribution is the \emph{form} of the low-rank
branch: instead of replacing cross-token aggregation with an absolute-coordinate surrogate, our method
keeps a genuine linear-attention branch and makes it compatible with relative rotary geometry by
combining RoPE outside the activation (\cref{sec:outside}) with rank-truncated reuse of the pre-trained
rotary schedule (\cref{sec:trunc}). The resulting design is supported by simple structural sanity checks in
Appendices~\ref{app:exact} and~\ref{app:inside}; these checks are intended to clarify the design
choice and are not claimed as theoretical contributions.

\Cref{tab:design-space} summarizes the component-level design space. The sparse local path and the use
of a global correction are shared with prior efficient-attention designs; our distinction lies in
retaining genuine cross-token global routing, preserving relative geometry in that path, and avoiding an
auxiliary positional module.

\begin{table}[t]
\centering
\small
\setlength{\tabcolsep}{3.5pt}
\caption{Capability-level comparison of representative efficient-attention methods. \cmark\ denotes that
the capability is explicitly present, \xmark\ denotes that it is absent, and \pmark\ denotes an
approximate realization.}
\label{tab:design-space}
\resizebox{\linewidth}{!}{%
\begin{tabular}{lcccccc}
\toprule
Method
& Sparse local path
& Global compensation
& Genuine global routing
& Relative-aware global path
& Reusable linear summary
& No extra positional parameters \\
\midrule
Pure sparse
& \cmark & \xmark & \xmark & \xmark & \xmark & \cmark \\
VSA
& \cmark & \xmark & \xmark & \xmark & \xmark & \cmark \\
SLA
& \cmark & \cmark & \cmark & \xmark & \cmark & \cmark \\
RoPeSLR
& \cmark & \cmark & \xmark & \pmark & \xmark & \xmark \\
\rowcolor{ourshade}
\textbf{RoLA (ours)}
& \cmark & \cmark & \cmark & \cmark & \cmark & \cmark \\
\bottomrule
\end{tabular}}
\end{table}

\subsection{Rotary-Positioned Low-Rank Linear Attention}
\label{sec:lowrank}

Let $q_p^{\mathrm{n}}=\mathrm{norm}_q(q_p)$ and $k_q^{\mathrm{n}}=\mathrm{norm}_k(k_q)$ denote the
main-path normalized, \emph{pre-RoPE} query and key representations, where $\mathrm{norm}_q$ and
$\mathrm{norm}_k$ denote the query/key normalization used in the pre-trained backbone. The low-rank branch
computes
\begin{align}
\qlr_p &= \bR_{\rr}(p)\,\SiLU\!\big(\Pq\, q_p^{\mathrm{n}}\big),
&
\klr_q &= \bR_{\rr}(q)\,\SiLU\!\big(\Pk\, k_q^{\mathrm{n}}\big),
\label{eq:qklr}
\end{align}
and
\begin{align}
C &= \sum_{q=1}^{L}\klr_q\, v_q^\top \;\in\;\R^{\rr\times \dd},
&
O^{\mathrm{lr}}_p &= \RMSNorm\!\big(\qlr_p^\top C\big),
\label{eq:ctx}
\end{align}
where $\Pq,\Pk\in\R^{\rr\times \dd}$ are head-specific projections, $v_q$ is the backbone value vector and
is not rotated, and $\bR_{\rr}(p)$ is the 3D RoPE rotation restricted to the first $\rr$ coordinates in the backbone's native ordering (\cref{sec:trunc}).

 The global context $C$ is computed \emph{once} per head, so the branch costs
$\bigo(L\rr \dd)$, i.e., it is linear in $L$. RMSNorm is applied after aggregation for numerical stability.
Our relative-position analysis concerns the pre-RMSNorm logits; RMSNorm acts as a stabilization layer on the
readout channels.

Two properties distinguish \cref{eq:qklr,eq:ctx} from both standard linear attention and the low rank
compensator used by \citet{liu2026ropeslr}. First, the rotation $\bR_{\rr}$ is applied to the low-rank feature
rather than to the full head dimension. Second, it is applied \emph{after} the nonlinearity $\SiLU$. We
analyze these two choices below. The full forward pass, including the sparse branch and gated fusion, is
given in \cref{alg:forward} (\cref{app:algorithm}). When $\rr$ is close to $\dd/2$, the leading matmul cost
of the low-rank branch can be comparable to that of a full-width linear branch (\cref{app:complexity}); the
purpose of the design is therefore not to reduce the linear-branch constant, but to obtain a
relative-position-aware global branch that remains reusable under high sparsity.

\subsection{Dimensionality Truncation of the Rotary Schedule}
\label{sec:trunc}

\noindent\textbf{Why truncate the rotary schedule?}
The projected feature
$\SiLU(\Pq q^{\mathrm{n}}_p)\in\R^{\rr}$
lives in an $\rr$-dimensional space with $\rr\ll \dd$; we use $\rr=64$ by default. Since RoPE acts on
2D rotary coordinate pairs, the low-rank branch can carry only $\rr/2$ such pairs and therefore cannot
reuse the full $\dd$-dimensional schedule. We truncate the pre-trained schedule to the first $\rr$ coordinates
under the backbone's RoPE coordinate ordering. These coordinates correspond to $\rr/2$ complete rotary
pairs, and we reuse the backbone's own $(\mathrm{freqs\_cos},\mathrm{freqs\_sin})$ tensors. This adds no
positional parameters and keeps the low-rank branch on the same geometric scale as the sparse branch and
the pre-trained backbone. Concretely, we take the first $\rr$ entries of the backbone's flattened RoPE
dimension, without reordering or re-learning frequencies. Because 3D RoPE partitions dimensions across axes,
this truncation inherits the backbone's per-axis allocation; the same construction applies to all axes, and
we report representative temporal and height slices in the mechanism experiments.

\begin{figure}[t]
\centering
\includegraphics[width=0.8\linewidth]{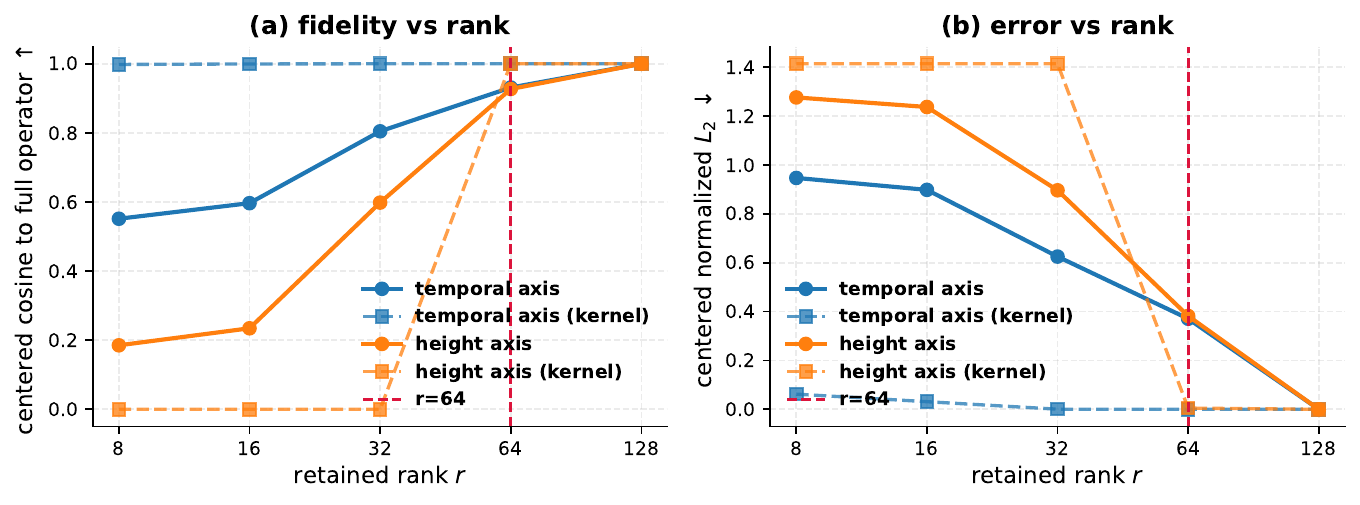}
\caption{Fidelity of the rank-$\rr$ truncated RoPE operator relative to the full
operator ($\rr=128$). The left panel reports centered cosine similarity, and the right panel reports
centered normalized $L_2$ error as $\rr$ increases. Solid lines use real-content operators, while dashed
lines use content-controlled positional kernels; blue and orange correspond to temporal and height
slices. Both metrics show that fidelity saturates by $\rr=64$, motivating our default truncation rank.}

\label{fig:truncation}
\end{figure}

\noindent\textbf{Why is $\rr=64$ sufficient in practice?}
The rank-$\rr$ truncation is motivated by a simple error intuition: if the retained rotary coordinates
capture most of the positional structure relevant to the global branch, the
discarded rotary coordinates contribute only a small residual score.
 \Cref{app:truncation-bound} gives an idealized nested-setting intuition for why such truncation can be benign. Because the trained $\mathbf{P}_q, \mathbf{P}_k$ are learned directly in the low-rank space, this intuition should not be read as a guarantee for the trained branch; the empirical fidelity sweep below is the operative evidence.
This intuition is consistent with the sparse-plus-smooth background view in
RoPeSLR~\citep{liu2026ropeslr}: the global branch mainly models a smooth and compressible
background, suggesting that its positional structure can be captured in a reduced rotary subspace.
In our implementation, we retain the first $\rr$ rotary coordinates, i.e., $\rr/2$ complete rotary coordinate pairs, in the backbone's native ordering; we do not reorder these coordinates by frequency. Empirically, Fig.~\ref{fig:truncation} evaluates this design choice: the centered cosine similarity
rises and the normalized $L_2$ error falls monotonically, with both metrics saturating by $\rr=64$.
Concretely, the temporal cosine improves from $0.55$ to $0.80$ and $0.93$ at $\rr{=}8,32,64$, while the
height cosine improves from $0.19$ to $0.60$ and $0.93$. We therefore choose $\rr=64$ as the default
operating point.

\subsection{RoPE Outside the Activation}
\label{sec:outside}

A key design choice in the low-rank branch concerns the \emph{order} of the rotation and the nonlinearity.
In RoLA, we first apply the low-rank projection and the pointwise nonlinearity, and then apply the
truncated RoPE rotation. This ordering is motivated by a simple separation of roles: the nonlinear low-rank
map should extract position-free content features, while RoPE should act only as a relative-position
operator. When rotation is applied after the nonlinearity, the global context can still be accumulated once
and reused by all queries; the positional effect appears only when a query reads out this context.
Appendix~\ref{app:exact} gives a short structural sanity check of this relative-position property.

If RoPE is instead applied before the nonlinearity, the rotation and the pointwise activation no longer
commute. As a result, position information becomes entangled with the content features before the global
aggregation, and the resulting key-side summary can depend on absolute positions rather than only on
relative offsets. Appendix~\ref{app:inside} gives a simple obstruction argument explaining why the
inside-activation ordering cannot generally admit an exact relative linear-attention factorization. This argument is intended as a design motivation rather than a learnability guarantee. Our
mechanism experiments in \cref{sec:mechanism} test these two predictions directly.

\begin{remark}[Contrast with absolute-PE injection]
The low-rank Fourier compensator of \citet{liu2026ropeslr} sidesteps the dilemma by removing cross-token routing
altogether and re-injecting position as a learnable absolute embedding. It learns to imitate relative decay
from absolute coordinates. RoLA instead keeps genuine cross-token routing and obtains relative
positional behavior through the outside-activation rotary design. The two designs are complementary points
in the design space; we compare them empirically in \cref{sec:ablation}.
\end{remark}

\subsection{Feature Alignment and Gated Fusion}
\label{sec:fusion}

Fusing localized sparse outputs $O^{\mathrm{s}}$ with the global low-rank output $O^{\mathrm{lr}}$ requires
variance stabilization. We normalize with RMSNorm \citep{zhang2019root} and combine the two branches through
a token-wise scalar gate
$$
g=\sigmoidf(X W_g + b_g)\in(0,1)^{L\times1},
$$
where $X$ denotes the input hidden states of the attention block and $W_g\in\R^{d_{\mathrm{model}}\times 1}$.
The gate is deliberately bottlenecked to one scalar per token, so it modulates the magnitude of the low-rank
branch rather than changing its channel-wise direction. The RMS of the sparse output provides a per-token
scale reference for stable fusion:
\begin{equation}
r_p=\sqrt{\tfrac{1}{\dd}\textstyle\sum_i (O^{\mathrm{s}}_{p,i})^2+\epsilon},
\qquad
O_p=\Big(\tfrac{O^{\mathrm{s}}_p}{r_p}+g_p\, O^{\mathrm{lr}}_p\Big)\,r_p .
\label{eq:gate}
\end{equation}
The sparse-output RMS is a natural per-token scale because it reflects the magnitude of the branch that
remains active at that token. Dividing by $r_p$ normalizes the sparse branch, while multiplying back restores
the original token scale after adding the gated global correction. The gate bias $b_g$ is initialized to a
negative value so that the branch starts close to the sparse baseline, allowing the low-rank contribution to
be introduced gradually during alignment.

\subsection{Implementation Details}
\label{sec:impl}

\noindent\textbf{Post-training strategy.}
\label{sec:posttrain}
RoLA is integrated into a pre-trained DiT by a single-stage full fine-tuning procedure. We initialize
the new parameters
$$
\Theta_{\mathrm{new}}=\{\Pq,\Pk,W_g,b_g\}
$$
with Kaiming initialization \citep{he2015delving} for the projections, zero for $W_g$, and a negative $b_g$
so that the module starts close to the sparse baseline. We then replace the attention module with RoLA
and jointly optimize all parameters, including the backbone and $\Theta_{\mathrm{new}}$, under the standard
diffusion flow-matching objective $\mathcal{L}_{\mathrm{task}}$ \citep{lipman2022flow}. The adaptation is
lightweight relative to pre-training: only the new branch parameters are trained from scratch, while the
backbone is updated with a small learning rate. The rationale for adopting this single-stage recipe, instead
of the two-stage strategy used in \citet{liu2026ropeslr}, is discussed in \cref{app:onestage}.

\noindent\textbf{Engineering: fused inference kernel.}
\label{sec:kernel}
At inference, the low-rank branch is evaluated using two fused Triton kernels. Because RoPE acts on
even/odd coordinate pairs, we split each projection weight into even and odd rows offline and process the
two halves as pure 2D matmuls, keeping the computation as dense matrix multiplications and avoiding
irregular indexing. Kernel 1 accumulates
$$
C_{\mathrm{even}}=\sum_q\klr^{\mathrm{even}}_q v_q^\top,
\qquad
C_{\mathrm{odd}}=\sum_q\klr^{\mathrm{odd}}_q v_q^\top
$$
over key/value tiles. Kernel 2 forms
$$
O^{\mathrm{lr}}_p=
\qlr^{\mathrm{even}}_p C_{\mathrm{even}}
+
\qlr^{\mathrm{odd}}_p C_{\mathrm{odd}}
$$
and applies the per-token RMSNorm. This implementation exactly reproduces \cref{eq:qklr,eq:ctx}. Detailed
FLOPs derivations and asymptotic complexity analysis are provided in \cref{app:complexity}.

\section{Experiments}
\label{sec:exp}

We evaluate RoLA at both the mechanism level and the generation level. Mechanism experiments test the
two design principles of the method, while generation experiments measure quality and efficiency on
full-scale video models.

\subsection{Mechanism-Level Validation}
\label{sec:mechanism}

\noindent\textbf{Protocol.}
Mechanism experiments use pre-trained Wan2.1-T2V-1.3B \citep{wan2025} weights with $\dd=128$ and $\rr=64$.
We extract normalized query/key features from real 480p, 81-frame denoising trajectories, average over heads
and denoising steps, and use the low-rank projections from the single-stage fine-tuned checkpoint. For all
positional metrics, we use pre-RMSNorm low-rank logits.

\begin{figure}[t]
\centering
\includegraphics[width=\linewidth]{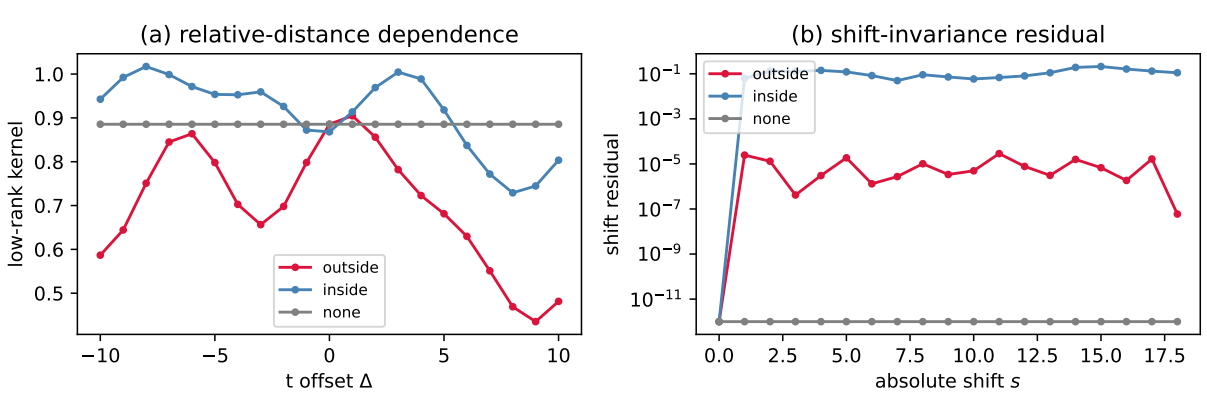}
\caption{Effect of RoPE placement in the low-rank branch. Panel (a) plots the low-rank kernel as a function
of relative offset, testing whether the branch is position-aware. Panel (b) plots the shift-invariance
residual under absolute translation, testing whether that dependence is purely relative. RoPE applied
\emph{outside} the activation is the only configuration that satisfies both conditions: non-trivial offset
dependence in (a) and near-zero shift residual in (b).}
\label{fig:exp4}
\end{figure}

\noindent\textbf{Low-rank branch encodes useful positional structure only with outside-RoPE
(\cref{fig:exp1}).}
We construct the low-rank attention map for a contiguous temporal slice and compare it with the
full-attention positional kernel under content-controlled inputs. The full kernel exhibits a clear banded
distance-decay structure. When RoPE is applied \emph{outside} the activation, the low-rank branch closely
matches this banded relative-position pattern. In contrast, the inside-activation variant remains structured
but visibly distorts the clean relative geometry of the full kernel. Removing RoPE collapses the map to a
nearly flat response, indicating that the branch becomes effectively position-blind. This result shows that
RoPE is necessary for the low-rank branch to encode useful positional structure, and that the outside
placement is the one that preserves the desired relative-position geometry.

\begin{wraptable}[6]{r}{0.65\textwidth}
    \centering
    \renewcommand{\arraystretch}{1.10}
    \setlength{\tabcolsep}{5.0pt}
    \caption{Positional fidelity by RoPE ordering.}
    \label{tab:mechanism}
    \resizebox{\linewidth}{!}{
    \begin{tabular}{lccc}
    \arrayrulecolor{tridentRed}
    \toprule
    \rowcolor{tridentHeader}
    \textbf{RoPE ordering} & \textbf{Shift residual} $\downarrow$ & \textbf{Offset range (position-aware)} & \textbf{Behavior} \\
    \arrayrulecolor{tridentOrange}
    \midrule
    No RoPE & $0.0$ & $0.0$ & position-blind \\
    Inside activation & $3.7\times10^{-1}$ & $0.47$ & leaks absolute position \\
    \rowcolor{tridentLight}
    \textcolor{tridentRed}{\textbf{Outside activation (ours)}} & $1.8\times10^{-5}$ & $0.47$ & \textcolor{tridentRed}{\textbf{relative / shift-invariant}} \\
    \arrayrulecolor{tridentRed}
    \bottomrule
    \end{tabular}
    }
    \arrayrulecolor{black}
\end{wraptable}

\noindent\textbf{RoPE placement: outside is both position-aware and shift-invariant
(\cref{fig:exp4,tab:mechanism}).}
For each of the three orderings---rotation \emph{outside} the activation (ours), rotation \emph{inside} the
activation, and \emph{no} rotation---we measure: (i) the relative-distance dependence of the low-rank logit
as a function of the temporal offset $\Delta=p-q$; and (ii) the shift-invariance residual
$$
|s(p{+}\delta,q{+}\delta)-s(p,q)|
$$
at a fixed offset as the pair is translated. An ordering that preserves relative-position structure should
exhibit both a non-trivial dependence on $\Delta$ and a near-zero shift residual. Only the outside-activation
ordering satisfies both conditions: it achieves a shift residual of $1.8{\times}10^{-5}$, which indicates
near-zero translation dependence, while maintaining an offset range of $0.47$ (\cref{tab:mechanism}).
Rotating \emph{inside} the activation produces similar offset variation but a much larger shift residual
($3.7{\times}10^{-1}$), consistent with the absolute-position leakage analyzed in \cref{app:inside}. The
no-RoPE branch has zero shift residual only trivially, because it is entirely position-blind with offset
range $0.0$; we therefore require both non-zero offset dependence and near-zero shift residual.

\begin{wraptable}{R}{0.55\textwidth}
    \centering
    \renewcommand{\arraystretch}{1.10}
    \setlength{\tabcolsep}{5.0pt}
    \caption{Rank-truncation fidelity.}
    \label{tab:trunc}
    \resizebox{\linewidth}{!}{
    \begin{tabular}{lccccc}
    \arrayrulecolor{tridentRed}
    \toprule
    \rowcolor{tridentHeader}
    \textbf{Retained rank $\rr$} & \textbf{8} & \textbf{16} & \textbf{32} & \cellcolor{tridentLight}\textcolor{tridentRed}{\textbf{64 (default)}} & \textbf{128} \\
    \arrayrulecolor{tridentOrange}
    \midrule
    Cosine to full, temporal $\uparrow$ & $0.55$ & $0.60$ & $0.80$ & \cellcolor{tridentLight}$0.93$ & $1.00$ \\
    Cosine to full, height $\uparrow$ & $0.19$ & $0.24$ & $0.60$ & \cellcolor{tridentLight}$0.93$ & $1.00$ \\
    \arrayrulecolor{tridentRed}
    \bottomrule
    \end{tabular}
    }
    \arrayrulecolor{black}
\end{wraptable}

\noindent\textbf{Rank truncation approximates the full operator (\cref{fig:truncation,tab:trunc}).}
We sweep the retained rotary rank $\rr\in\{8,16,32,64,128\}$ and measure how faithfully the rank-$\rr$
truncated operator approximates the full ($\rr=128$) relative-position operator on real Wan Q/K features.
For an axis-specific operator slice $A_r$ and the full-rank reference $A_{128}$, we first subtract the
entry-wise mean to obtain centered operators $\widetilde A_r$ and $\widetilde A_{128}$.
\Cref{fig:truncation} reports two fidelity metrics: the centered cosine similarity and the normalized
$L_2$ error. The centered cosine saturates by $\rr=64$: on the temporal axis it improves from $0.55$ to
$0.80$ and $0.93$ at $\rr{=}8,32,64$, reaching $1.0$ at $\rr{=}128$; on the height axis it improves from
$0.19$ to $0.60$ and $0.93$, also reaching $1.0$ at $\rr{=}128$. The normalized $L_2$ error falls to near
zero at the same point. Both metrics therefore identify $\rr=64$ as a sufficient operating point for
approximating the full relative-position operator on real content, consistent with the truncation intuition
in \cref{app:truncation-bound}.

\subsection{Generation Quality}
\label{sec:generation}

\noindent\textbf{Models and datasets.}
We evaluate RoLA on Wan2.1-1.3B and Wan2.1-14B~\citep{wan2025}. For lightweight adaptation of the
newly introduced branch, we use a curated subset of OpenVid-1M
\citep{nan2024openvid} for both models. Quantitative evaluation follows the official VBench protocol.

\noindent\textbf{Baselines.}
We compare with dense full attention, VSA~\citep{zhang2026faster}, the sparse--linear method SLA
\citep{zhang2025sla}, and the sparse low-rank predecessor RoPeSLR \citep{liu2026ropeslr}. All baselines are
implemented based on their official implementations and adapted to the same evaluation protocol when
necessary. Sparse-based baselines use the same target sparsity when applicable, and all latency comparisons
are measured under the same evaluation setting reported in each table or figure caption.
For trainable baselines such as SLA and RoPeSLR, we use the same adaptation data, batch size, and number of
optimization steps as RoLA; where a method's original implementation requires a specific sparse routing or
training recipe, we keep that method-specific design and otherwise follow the official code path.
We do not include SANA-Video or SANA-Video 2.0 here, since they follow dedicated linear/hybrid-linear
backbone training paradigms rather than the sparse post-training setting studied in this paper.

\noindent\textbf{Metrics.}
We report five VBench dimensions~\citep{huang2023vbench}: Subject Consistency, Motion Smoothness,
Background Consistency, Imaging Quality, and Aesthetic Quality. For each reported VBench dimension, we
follow the official VBench evaluation protocol, using the corresponding official prompt set, evaluator, and
scoring procedure consistently for all methods. We choose these dimensions because they are sensitive to
temporal stability, background consistency, and visual quality, which are often affected by aggressive
sparsification. Efficiency is measured by DiT end-to-end latency, attention sparsity, and theoretical
attention cost. Here latency refers to end-to-end runtime of the DiT denoising component, excluding text
encoding and VAE stages. Sparsity denotes the fraction of query-key attention interactions removed by the
sparse attention engine relative to full attention. We additionally report PSNR, SSIM, and LPIPS as
complementary frame-level fidelity and perceptual metrics, computed against the dense full-attention
generated video under the same prompt and seed; the full row is therefore used as the reference and left
blank.

\begin{table*}[t]
\centering
\scriptsize
\renewcommand{\arraystretch}{1.08}
\setlength{\tabcolsep}{3.8pt}
\caption{Quantitative comparison on Wan2.1-1.3B and Wan2.1-14B benchmarks. We use NVIDIA H100 80G GPUs for
all experiments. Wan2.1-1.3B uses 480p inputs. Wan2.1-14B uses 720p inputs with 81 frames. Latency is
measured for the DiT denoising component under the VBench evaluation setting and is therefore not directly
comparable to the RTX 5090 deployment measurements in Figure~\ref{fig:efficiency}.}
\label{tab:main}
\resizebox{\textwidth}{!}{%
\begin{tabular}{llcccccccccc}
\arrayrulecolor{tridentRed}
\toprule
\rowcolor{tridentHeader}
& & \multicolumn{5}{c}{\textbf{VBench Quality}} & \multicolumn{3}{c}{\textbf{Additional Metrics}} & \multicolumn{2}{c}{\textbf{Efficiency}} \\
\arrayrulecolor{tridentRed}
\hhline{~~----------}
\rowcolor{tridentHeader}
\multirow{-2}{*}{\textbf{Model}} & \multirow{-2}{*}{\textbf{Method}}
& \textbf{SC}$\uparrow$ & \textbf{MS}$\uparrow$ & \textbf{BC}$\uparrow$ & \textbf{IQ}$\uparrow$ & \textbf{AQ}$\uparrow$
& \textbf{PSNR}$\uparrow$ & \textbf{SSIM}$\uparrow$ & \textbf{LPIPS}$\downarrow$
& \textbf{Latency}$\downarrow$ & \textbf{Sparsity}$\uparrow$ \\
\arrayrulecolor{tridentOrange}
\midrule
\multirow{5}{*}{Wan2.1-1.3B}
& Full & 0.9403 & 0.9557 & \best{0.9498} & \best{0.6510} & \best{0.6088} & -- & -- & -- & 50.9s & $0\%$ \\
& VSA & 0.9364 & 0.9757 & 0.9421 & 0.6328 & 0.5629 & 23.05 & 0.819 & 0.149 & 34.3s & $90\%$ \\
& SLA & 0.9401 & 0.9782 & 0.9486 & 0.6343 & 0.5918 & 24.17 & 0.844 & 0.130 & 31.2s & $90\%$ \\
& RoPeSLR & 0.9512 & 0.9789 & 0.9485 & 0.6481 & 0.6010 & 25.01 & 0.843 & 0.129 & 30.4s & $90\%$ \\
& \cellcolor{tridentLight}\textcolor{tridentRed}{\textbf{RoLA (Ours)}} & \cellcolor{tridentLight}\best{0.9521} & \cellcolor{tridentLight}\best{0.9812} & \cellcolor{tridentLight}0.9490 & \cellcolor{tridentLight}0.6492 & \cellcolor{tridentLight}0.6068 & \cellcolor{tridentLight}\best{25.67} & \cellcolor{tridentLight}\best{0.860} & \cellcolor{tridentLight}\best{0.121} & \cellcolor{tridentLight}\best{21.7s} & \cellcolor{tridentLight}$90\%$ \\
\arrayrulecolor{tridentOrange}
\midrule
\multirow{5}{*}{Wan2.1-14B}
& Full & \best{0.9698} & 0.9850 & \best{0.9681} & \best{0.6996} & \best{0.6319} & -- & -- & -- & 1075.7s & $0\%$ \\
& VSA & 0.9010 & 0.9788 & 0.9482 & 0.6012 & 0.5998 & 24.27 & 0.843 & 0.130 & 717.1s & $90\%$ \\
& SLA & 0.9481 & 0.9802 & 0.9511 & 0.6398 & 0.6052 & 26.02 & 0.878 & 0.124 & 601.0s & $90\%$ \\
& RoPeSLR & 0.9601 & 0.9842 & 0.9528 & 0.6627 & 0.6102 & 26.33 & 0.880 & 0.118 & 598.3s & $90\%$ \\
& \cellcolor{tridentLight}\textcolor{tridentRed}{\textbf{RoLA (Ours)}} & \cellcolor{tridentLight}0.9629 & \cellcolor{tridentLight}\best{0.9853} & \cellcolor{tridentLight}0.9537 & \cellcolor{tridentLight}0.6783 & \cellcolor{tridentLight}0.6294 & \cellcolor{tridentLight}\best{28.65} & \cellcolor{tridentLight}\best{0.891} & \cellcolor{tridentLight}\best{0.108} & \cellcolor{tridentLight}\best{409.3s} & \cellcolor{tridentLight}$90\%$ \\
\arrayrulecolor{tridentRed}
\bottomrule
\end{tabular}
}%
\arrayrulecolor{black}
\end{table*}

\noindent\textbf{Quality evaluation.}
As shown in Table~\ref{tab:main}, RoLA reduces computational overhead relative to dense full attention
while remaining competitive on the reported VBench dimensions. Unlike RoPeSLR~\citep{liu2026ropeslr}, whose
global module is a low-rank Fourier compensator that imitates relative decay with an absolute positional
embedding, RoLA injects relative-position structure
without additional positional parameters (\cref{sec:outside}), which is consistent with its improved
mechanism-level fidelity at matched sparsity (\cref{sec:mechanism}). In our evaluation, the VBench
degradation relative to full attention is small. The low-rank branch is intended to recover global context
discarded by pure sparse baselines such as VSA, and at matched sparsity RoLA achieves VBench scores
closer to dense attention on the reported dimensions. Overall, these results support a favorable
quality--efficiency trade-off within the evaluated settings.

\noindent\textbf{Fine-tuning control.}
To isolate the effect of the attention algorithm from adaptation itself, we additionally evaluate the
full-attention model after applying the same fine-tuning protocol as RoLA. The corresponding two-row
comparison is provided in Appendix Table~\ref{tab:abl_full_ft}; it uses the same adaptation data, batch
size, optimization steps, and evaluation protocol, while changing only whether the attention block is
kept dense or replaced by RoLA.

\subsection{Ablation Study}
\label{sec:ablation}
\begin{wraptable}{R}{0.52\textwidth}
    \centering
    \renewcommand{\arraystretch}{1.10}
    \setlength{\tabcolsep}{5.0pt}
    \caption{Ablation on retained rotary rank $\rr$.}
    \label{tab:abl_rank}
    \resizebox{\linewidth}{!}{
    \begin{tabular}{lccccc}
    \arrayrulecolor{tridentRed}
    \toprule
    \rowcolor{tridentHeader}
    \textbf{Rank $\rr$} & \textbf{SC}$\uparrow$ & \textbf{MS}$\uparrow$ & \textbf{BC}$\uparrow$ & \textbf{IQ}$\uparrow$ & \textbf{AQ}$\uparrow$ \\
    \arrayrulecolor{tridentOrange}
    \midrule
    8 & 0.8788 & 0.8715 & 0.9170 & 0.5574 & 0.5640 \\
    16 & 0.8910 & 0.8861 & 0.9298 & 0.5772 & 0.5601 \\
    32 & 0.9398 & 0.9540 & 0.9394 & 0.6138 & 0.5990 \\
    \rowcolor{tridentLight}
    \textcolor{tridentRed}{\textbf{64 (default)}} & 0.9521 & 0.9812 & 0.9490 & \best{0.6492} & 0.6068 \\
    128 & \best{0.9540} & \best{0.9841} & \best{0.9525} & 0.6479 & \best{0.6099} \\
    \arrayrulecolor{tridentRed}
    \bottomrule
    \end{tabular}
    }
    \arrayrulecolor{black}
\end{wraptable}
We ablate the central design axes to isolate their contributions. All ablations use Wan2.1-1.3B at matched
$90\%$ sparsity.

\noindent\textbf{Controlled global-branch comparison.}
To separate the effect of the global branch from the sparse attention pattern, we fix the sparse engine,
pattern, and target sparsity, and replace only the global branch. The resulting comparison is reported in
Appendix Table~\ref{tab:abl_global}. This experiment directly tests whether the gains come from the
rotary low-rank construction rather than from a different sparse routing pattern.

\noindent\textbf{Sparsity-quality trade-off.}
We also evaluate RoLA across multiple sparsity levels and report the selected VBench dimensions in
Appendix Table~\ref{tab:abl_sparsity}. This sweep tests whether the proposed global branch remains useful
as the sparse branch becomes increasingly aggressive.

\begin{wraptable}{R}{0.38\textwidth}
    \centering
    \renewcommand{\arraystretch}{1.10}
    \setlength{\tabcolsep}{5.0pt}
    \caption{Activation ablation. Projections are
    retrained per activation; the sparse-only baseline rel-$L_2$
    is $0.612$.}
    \label{tab:abl_activation}
    \begin{tabular}{lc}
    \arrayrulecolor{tridentRed}
    \toprule
    \rowcolor{tridentHeader}
    \textbf{Activation} & \textbf{Align rel-$L_2$} $\downarrow$ \\
    \arrayrulecolor{tridentOrange}
    \midrule
    Identity (linear) & $0.525$ \\
    Tanh & $0.525$ \\
    ReLU & $0.534$ \\
    GELU & $0.522$ \\
    \rowcolor{tridentLight}
    \textcolor{tridentRed}{\textbf{SiLU (default)}} & $\best{0.521}$ \\
    \arrayrulecolor{tridentRed}
    \bottomrule
    \end{tabular}
    \arrayrulecolor{black}
\end{wraptable}

\noindent\textbf{RoPE placement (structurally checked in \cref{app:exact}).}
We compare no rotation, rotation \emph{inside} the activation, and rotation \emph{outside} the activation.
\Cref{tab:mechanism} shows that only the outside ordering preserves relative-position structure; the inside
ordering leaks absolute position despite remaining position-aware.

\noindent\textbf{Rotary truncation rank (\cref{fig:truncation}).}
We sweep $\rr\in\{8,16,32,64,128\}$. Operator fidelity saturates at $\rr=64$ (\cref{tab:trunc}), and
\cref{tab:abl_rank} shows that downstream quality follows the same trend, with clear gains up to $\rr=64$
and only marginal change beyond it.

\noindent\textbf{Activation function.}
Here the alignment error \emph{Align rel-$L_2$} measures how well the sparse--low-rank output reconstructs
the dense full-attention output: for each activation we retrain the low-rank projections and report the
relative $L_2$ distance $\lVert O^{\mathrm{fused}}-O^{\mathrm{full}}\rVert_2/\lVert O^{\mathrm{full}}\rVert_2$
between the fused attention output and the full-attention reference on real Wan Q/K/V features, averaged over
tokens and heads (lower is better). As a reference, the sparse-only baseline without any low-rank branch
attains rel-$L_2=0.612$.
We ablate SiLU against GELU, ReLU, Tanh, and identity. Table~\ref{tab:abl_activation} shows that SiLU
achieves the lowest alignment error among the tested activations. Identity and Tanh underperform,
indicating that a smooth nonlinearity is beneficial for reconstructing the low-rank branch.

\section{Conclusion}
\label{sec:conclusion}

We present RoLA, a rotary-positioned low-rank attention branch for efficient video Diffusion
Transformers. It addresses the compatibility issue between reusable linear aggregation and 3D relative
rotary geometry in the low-rank global branch. By applying rank-truncated 3D RoPE outside the nonlinear
activation, the branch keeps a reusable linear-time global summary while injecting relative positional
structure without auxiliary absolute-position parameters. Mechanistic analysis and empirical evaluations
indicate that RoLA achieves a favorable quality--efficiency trade-off under high sparsity, retaining
generation fidelity while substantially reducing computational overhead.

\noindent\textbf{Limitations.}
Our current evaluation focuses on two Wan models. The appendix contains structural sanity checks and idealized design intuitions (Propositions \ref{prop:exact}, \ref{prop:truncation-bound}, and \ref{prop:inside}); they are not theoretical contributions, do not provide end-to-end generation guarantees, and are included only to motivate and verify the design choice.

\section{Acknowledgments}
This work was supported by Alibaba Group through Alibaba Research Intern Program.

\clearpage
\bibliographystyle{plainnat}
\bibliography{references}

\clearpage
\clearpage
\beginappendix

\section{Full Forward Pass}
\label{app:algorithm}

\Cref{alg:forward} gives the complete RoLA attention forward pass for a single block, combining the
sparse branch, the rotary-positioned low-rank branch, and the distribution-aware gated fusion of
\cref{sec:method}. All new parameters are jointly optimized with the backbone.

\begin{algorithm}[h]
\caption{RoLA attention forward pass (single block, per head $h$).}
\label{alg:forward}
\begin{algorithmic}[1]
\Require hidden states $X\in\R^{L\times d_{\mathrm{model}}}$; attention module with projections
$W_q,W_k,W_v$, pre-attention norms $\mathrm{norm}_q,\mathrm{norm}_k$, output projection $W_o$; new
parameters $\Pq,\Pk\in\R^{\rr\times\dd}$, gate $W_g\in\R^{d_{\mathrm{model}}\times1}$, gate bias $b_g$;
full rotary schedule $\bR(\cdot)$; rank-truncated schedule $\bR_{\rr}(\cdot)$; sparse engine
$\mathrm{SparseAttn}$ (e.g., VMoBA/SLA) with topk ratio
\Ensure attention output $O\in\R^{L\times d_{\mathrm{model}}}$
\State $Q,K,V \gets W_q X,\; W_k X,\; W_v X$ \Comment{QKV projection}
\State $Q^{\mathrm{n}},K^{\mathrm{n}} \gets \mathrm{norm}_q(Q),\; \mathrm{norm}_k(K)$
\Comment{pre-RoPE normalized query/key, reshaped to $[L,\dd]$ per head}
\Statex \textbf{-- Sparse branch (high-energy spikes) --}
\State $Q^{\mathrm{r}},K^{\mathrm{r}} \gets \bR(Q^{\mathrm{n}}),\; \bR(K^{\mathrm{n}})$
\Comment{full-dimensional RoPE on the main path}
\State $O^{\mathrm{s}} \gets \mathrm{SparseAttn}(Q^{\mathrm{r}},K^{\mathrm{r}},V)$
\Comment{$\bigo\!\big(L^2(1{-}S)\dd\big)$}
\Statex \textbf{-- Low-rank branch (smooth global background), RoPE outside activation --}
\State $\qlr \gets \bR_{\rr}\!\big(\SiLU(Q^{\mathrm{n}}\Pq^\top)\big)$,\quad
       $\klr \gets \bR_{\rr}\!\big(\SiLU(K^{\mathrm{n}}\Pk^\top)\big)$
\Comment{activate, then rotate; $\qlr,\klr\in\R^{L\times\rr}$}
\State $C \gets \klr^\top V\in\R^{\rr\times\dd}$
\Comment{global context, computed once per head: $\bigo(L\rr\dd)$}
\State $O^{\mathrm{lr}} \gets \RMSNorm\!\big(\qlr\, C\big)$
\Comment{per-query readout: $\bigo(L\rr\dd)$}
\Statex \textbf{-- Distribution-aware gated fusion --}
\State $g \gets \sigmoidf\!\big(X W_g + b_g\big)\in(0,1)^{L\times1}$
\Comment{token-wise scalar gate}
\State $r \gets \sqrt{\tfrac{1}{\dd}\textstyle\sum_i (O^{\mathrm{s}}_{:,i})^2+\epsilon}$
\Comment{per-token RMS of the sparse output}
\State $O^{\mathrm{fused}} \gets \big(O^{\mathrm{s}}/r + g\, O^{\mathrm{lr}}\big)\,r$
\Comment{RMS-normalized gated combination, Eq.~\ref{eq:gate}}
\State $O \gets W_o\,\mathrm{concat}_h\!\big(O^{\mathrm{fused}}\big)$
\Comment{merge heads, output projection}
\State \Return $O$
\end{algorithmic}
\end{algorithm}

\section{Formal Sanity Checks and Design Rationale}
\label{app:proofs}

This appendix collects simple formal checks that motivate and sanity-check the design of RoLA. We do not present them as novel theoretical contributions or end-to-end performance guarantees. \Cref{prop:exact} is a \emph{structural sanity check}: by a direct application of the RoPE group property, the outside-activation branch satisfies shift-invariance and $\bigo(L)$ reuse; its role is to verify that this ordering meets the two desired structural requirements simultaneously. \Cref{prop:truncation-bound} is an \emph{idealized design intuition}: it bounds the truncation error in a nested full-dimensional idealization. Because $\mathbf{P}_k,\mathbf{P}_q$ are learned directly in the low-rank space, this intuition motivates rather than guarantees the trained branch, whose fidelity is measured empirically (\cref{sec:trunc}). \Cref{prop:inside} is a \emph{structural obstruction argument}: inside-activation SiLU admits no \emph{exact} relative factorization; this rules out exactness, not approximate learnability, and whether a trained inside branch approximates relative behavior is an empirical question answered in \cref{sec:mechanism}.

Throughout this appendix, $\bR_{\rr}(p)\in\R^{\rr\times \rr}$ denotes the 3D RoPE rotation restricted to the
first $\rr$ rotary coordinates in the backbone's native coordinate ordering, i.e., the block-diagonal matrix formed by the retained complete 2D rotary pairs,
split across the three spatiotemporal axes. This notation only specifies the retained coordinate indices and does not make any frequency-ordering assumption. It is orthogonal,

$$
\bR_{\rr}(p)^\top\bR_{\rr}(p)=I,
$$
and satisfies the group property
$$
\bR_{\rr}(p)^\top\bR_{\rr}(q)=\bR_{\rr}(q-p),
$$
which is inherited blockwise from the 2D rotation identity
$$
\mathrm{Rot}(\theta p)^\top\mathrm{Rot}(\theta q)=\mathrm{Rot}(\theta(q-p)).
$$

\subsection{Relative-Position Structure of the Outside-Activation Branch}
\label{app:exact}

The following sanity check verifies the main reason for applying RoPE after the low-rank nonlinearity: if the nonlinear feature map is position-free, the rotary rotation can be moved into a relative offset without destroying the reusable linear aggregation. The proof is a direct application of the orthogonality and group property of RoPE. We emphasize that this is a structural sanity check rather than a novel theoretical result: it confirms that the outside ordering simultaneously satisfies the two requirements that \cref{prop:inside} shows cannot hold jointly under the inside ordering.

\begin{proposition}[Structural sanity check: relative-position structure of the outside-activation branch]

\label{prop:exact}
Let $\phi(\cdot)=\SiLU(\Pq\,\cdot)$ and $\psi(\cdot)=\SiLU(\Pk\,\cdot)$ be the pointwise low-rank feature
maps. With RoPE applied \emph{outside} the activation as in \cref{eq:qklr}, the low-rank pairwise
interaction is a function of the relative offset:
\begin{equation}
\qlr_p^\top\klr_q
=
\big(\bR_{\rr}(p)\phi(q_p)\big)^\top
\big(\bR_{\rr}(q)\psi(k_q)\big)
=
\phi(q_p)^\top\,\bR_{\rr}(q-p)\,\psi(k_q).
\end{equation}
Meanwhile, the global context $C=\sum_q\klr_q v_q^\top$ remains a single reusable sum. Hence the branch is
simultaneously shift-invariant and $\bigo(L)$.
\end{proposition}

\begin{proof}
Write the low-rank features in \cref{eq:qklr} as
$$
\qlr_p=\bR_{\rr}(p)\phi(q_p),
\qquad
\klr_q=\bR_{\rr}(q)\psi(k_q),
$$
where
$$
\phi(q_p)=\SiLU(\Pq q^{\mathrm{n}}_p),
\qquad
\psi(k_q)=\SiLU(\Pk k^{\mathrm{n}}_q)
$$
are position-free: they depend only on token content, not on the indices $p$ or $q$.

\emph{Step 1: relative interaction.}
For any query--key pair,
$$
\qlr_p^\top\klr_q
=
\big(\bR_{\rr}(p)\phi(q_p)\big)^\top
\big(\bR_{\rr}(q)\psi(k_q)\big)
=
\phi(q_p)^\top\,\bR_{\rr}(p)^\top\bR_{\rr}(q)\,\psi(k_q)
=
\phi(q_p)^\top\,\bR_{\rr}(q-p)\,\psi(k_q),
$$
where the last equality uses orthogonality and the group property. The right-hand side depends on the
positions only through the relative offset $q-p$, establishing shift-invariance of the low-rank logit.

\emph{Step 2: $\bigo(L)$ associativity is preserved.}
The global context is
$$
C=\sum_{q=1}^{L}\klr_q v_q^\top
=
\sum_{q=1}^{L}\bR_{\rr}(q)\psi(k_q)v_q^\top
\in\R^{\rr\times \dd}.
$$
This is a single sum that does not depend on the query index $p$; it can therefore be computed once at cost
$\bigo(L\rr \dd)$. The per-query output
$$
O^{\mathrm{lr}}_p=\RMSNorm(\qlr_p^\top C)
$$
costs $\bigo(\rr \dd)$ per token, so the branch is $\bigo(L\rr \dd)$ overall. The rotation $\bR_{\rr}(q)$ is
absorbed into $C$ before the sum, and the query-side rotation $\bR_{\rr}(p)$ is applied after $C$; neither
obstructs associativity, because both are linear operators acting outside the position-free feature maps.
This is what simultaneously yields relative-position structure and linear cost.
\end{proof}

\subsection{Truncation-Error Intuition for the Rotary Low-Rank Branch}
\label{app:truncation-bound}

\paragraph{Notation and idealized nested view.}
Let \(\bR(p)\in\R^{\dd\times\dd}\) be the full 3D RoPE matrix, block-diagonal with \(2\times2\) rotation
blocks. Let \(\rr=2s\le \dd\) be an even truncation rank, and assume that the first \(\rr\) coordinates in the backbone's native ordering correspond to the first \(s\) complete \(2\times2\) blocks.
Define
\[
P_{\le\rr}=\operatorname{diag}(I_{\rr},\;0_{(\dd-\rr)\times(\dd-\rr)}),
\qquad
P_{>\rr}=I_{\dd}-P_{\le\rr}.
\]
For a vector \(u\in\R^{\dd}\), write
\[
u_{\le\rr}=P_{\le\rr}u,\qquad u_{>\rr}=P_{>\rr}u.
\]
The truncated RoPE matrix is
\[
\bR_{\rr}(p)=P_{\le\rr}\bR(p)P_{\le\rr}.
\]

To make the truncation intuition precise, we consider an idealized nested setting. Let \(\Phi(q_p)\) and
\(\Psi(k_q)\) be position-free full-dimensional feature maps, and suppose the rank-\(\rr\) branch retains
the first \(\rr\) rotary coordinates of these maps in the backbone's native ordering. In the actual implementation, \(\Pq\) and \(\Pk\) are
learned directly in the low-dimensional space. Therefore, the following statement is not a theoretical
contribution and should not be read as a strict guarantee for the trained branch; it is an idealized
design intuition. The empirical operator-fidelity results in the main text are the operative evidence.

\begin{assumption}[Concentration in the retained rotary subspace]
\label{ass:rotfreq}
For the global background branch of a pre-trained video DiT, the residual feature components outside the retained rotary subspace are small.
Specifically, there exist constants \(\epsilon_\Phi(\rr),\epsilon_\Psi(\rr)>0\) such that
\[
\sup_{p}\left\|P_{>\rr}\Phi(q_p)\right\|_2
\le \epsilon_\Phi(\rr),
\qquad
\sup_{q}\left\|P_{>\rr}\Psi(k_q)\right\|_2
\le \epsilon_\Psi(\rr).
\]
For 3D-RoPE video DiTs, this concentration is consistent with the sparse-plus-smooth background decomposition
identified in the low-rank Fourier compensation analysis of RoPeSLR~\citep{liu2026ropeslr}. RoPE uses a geometrically spaced rotary frequency schedule~\citep{su2021roformer}; our implementation retains the first \(\rr\) rotary coordinates in the backbone's native ordering, and this choice is justified by the empirical fidelity sweep rather than by a frequency-ordering assumption. Fig.~\ref{fig:truncation} validates the sufficiency of this retained subspace at
\(\rr=64\).
\end{assumption}

\begin{proposition}[Idealized truncation intuition for a nested rotary feature map]
\label{prop:truncation-bound}
Define the full and rank-\(\rr\) truncated pairwise scores as
\[
s_{\mathrm{full}}(p,q)
=
\Phi(q_p)^\top \bR(q-p)\Psi(k_q),
\]
and
\[
s_{\rr}(p,q)
=
\Phi_{\le\rr}(q_p)^\top
\bR_{\rr}(q-p)
\Psi_{\le\rr}(k_q).
\]
Then for all tokens \(p,q\),
\[
\left|s_{\mathrm{full}}(p,q)-s_{\rr}(p,q)\right|
\le
\left\|P_{>\rr}\Phi(q_p)\right\|_2
\left\|P_{>\rr}\Psi(k_q)\right\|_2.
\]
In particular, under Assumption~\ref{ass:rotfreq},
\[
\left|s_{\mathrm{full}}(p,q)-s_{\rr}(p,q)\right|
\le
\epsilon_\Phi(\rr)\,\epsilon_\Psi(\rr).
\]
\end{proposition}

\begin{proof}
Because \(P_{\le\rr}\) selects an integer number of complete \(2\times2\) RoPE blocks, the subspaces
\[
U_{\le\rr}=\operatorname{range}(P_{\le\rr}),\qquad
U_{>\rr}=\operatorname{range}(P_{>\rr})
\]
are invariant under \(\bR(p)\). Therefore,
\[
P_{\le\rr}\bR(p)P_{>\rr}=0,\qquad
P_{>\rr}\bR(p)P_{\le\rr}=0.
\]
In other words, \(\bR(p)\) preserves the split between the retained rotary subspace and the discarded rotary subspace.

Now decompose the features as
\[
\Phi(q_p)=\Phi_{\le\rr}(q_p)+\Phi_{>\rr}(q_p),\qquad
\Psi(k_q)=\Psi_{\le\rr}(k_q)+\Psi_{>\rr}(k_q).
\]
Then
\[
\begin{aligned}
s_{\mathrm{full}}(p,q)
&=
\left(\Phi_{\le\rr}(q_p)+\Phi_{>\rr}(q_p)\right)^\top
\bR(q-p)
\left(\Psi_{\le\rr}(k_q)+\Psi_{>\rr}(k_q)\right) \\
&=
\Phi_{\le\rr}(q_p)^\top
\bR(q-p)
\Psi_{\le\rr}(k_q)
+
\Phi_{\le\rr}(q_p)^\top
\bR(q-p)
\Psi_{>\rr}(k_q)
\\
&\quad+
\Phi_{>\rr}(q_p)^\top
\bR(q-p)
\Psi_{\le\rr}(k_q)
+
\Phi_{>\rr}(q_p)^\top
\bR(q-p)
\Psi_{>\rr}(k_q).
\end{aligned}
\]
By the invariance of \(U_{\le\rr}\) and \(U_{>\rr}\), the two cross terms vanish:
\[
\Phi_{\le\rr}(q_p)^\top \bR(q-p)\Psi_{>\rr}(k_q)=0,
\qquad
\Phi_{>\rr}(q_p)^\top \bR(q-p)\Psi_{\le\rr}(k_q)=0.
\]
Hence
\[
s_{\mathrm{full}}(p,q)
=
\Phi_{\le\rr}(q_p)^\top
\bR(q-p)
\Psi_{\le\rr}(k_q)
+
\Phi_{>\rr}(q_p)^\top
\bR(q-p)
\Psi_{>\rr}(k_q).
\]
On the first term, \(\Phi_{\le\rr}\) and \(\Psi_{\le\rr}\) live only on the first \(\rr\) coordinates, so

\[
\Phi_{\le\rr}(q_p)^\top \bR(q-p)\Psi_{\le\rr}(k_q)
=
\Phi_{\le\rr}(q_p)^\top \bR_{\rr}(q-p)\Psi_{\le\rr}(k_q)
=
s_{\rr}(p,q).
\]
Therefore,
\[
s_{\mathrm{full}}(p,q)-s_{\rr}(p,q)
=
\Phi_{>\rr}(q_p)^\top
\bR(q-p)
\Psi_{>\rr}(k_q).
\]
Since \(\bR(q-p)\) is orthogonal,
\[
\left|s_{\mathrm{full}}(p,q)-s_{\rr}(p,q)\right|
=
\left|
\Phi_{>\rr}(q_p)^\top
\bR(q-p)
\Psi_{>\rr}(k_q)
\right|.
\]
By Cauchy--Schwarz,
\[
\left|s_{\mathrm{full}}(p,q)-s_{\rr}(p,q)\right|
\le
\left\|\Phi_{>\rr}(q_p)\right\|_2
\left\|\bR(q-p)\Psi_{>\rr}(k_q)\right\|_2
=
\left\|\Phi_{>\rr}(q_p)\right\|_2
\left\|\Psi_{>\rr}(k_q)\right\|_2.
\]
The last equality uses orthogonality of \(\bR(q-p)\).

Finally, under Assumption~\ref{ass:rotfreq}, the right-hand side is bounded by
\(\epsilon_\Phi(\rr)\epsilon_\Psi(\rr)\), which gives the stated bound.
\end{proof}

\subsection{Obstruction for Inside-Activation Rotary Factorization}
\label{app:inside}

The following obstruction argument explains why applying RoPE before the pointwise nonlinearity is structurally unfavorable: a pointwise nonlinearity does not commute with a rotation, so absolute position can become entangled with content before global aggregation. This argument is intended as a design motivation; it rules out exact factorization, not approximate learnability.

\begin{proposition}[No exact relative factorization under inside-activation \(\SiLU\)]
\label{prop:inside}
Let \(\rr\ge 2\), and let \(\sigma=\SiLU\) be applied coordinatewise. Define the inside-activation low-rank
features
\[
\qlr^{\prime}_p=\sigma\bigl(\bR_{\rr}(p)\Pq q^{\mathrm{n}}_p\bigr),
\qquad
\klr^{\prime}_q=\sigma\bigl(\bR_{\rr}(q)\Pk k^{\mathrm{n}}_q\bigr).
\]
Then, in the idealized continuous-position setting, there do not exist functions \(g,h\) such that for all
\(p,q\in\R^3\) and all content vectors \(q^{\mathrm{n}}_p,k^{\mathrm{n}}_q\in\R^{\rr}\),
\[
(\qlr^{\prime}_p)^\top\klr^{\prime}_q
=
g(q^{\mathrm{n}}_p,k^{\mathrm{n}}_q)\,h(q-p).
\]
In particular, if such a factorization existed, the global context
\[
C'=\sum_q\klr^{\prime}_q v_q^\top
\]
would be shift-invariant; the proposition shows that this is impossible.
\end{proposition}

\begin{proof}
We prove this by contradiction.

Assume that such a factorization exists:
\[
\sigma\bigl(\bR_{\rr}(p)x\bigr)^\top
\sigma\bigl(\bR_{\rr}(q)y\bigr)
=
g(x,y)\,h(q-p)
\quad
\forall p,q\in\R^3,\quad\forall x,y\in\R^{\rr}.
\]

\noindent\textbf{Step 1: Norm condition from the factorization.}
Set \(q=p\) and \(y=x\). Then the left-hand side becomes
\[
\sigma\bigl(\bR_{\rr}(p)x\bigr)^\top
\sigma\bigl(\bR_{\rr}(p)x\bigr)
=
\left\|\sigma\bigl(\bR_{\rr}(p)x\bigr)\right\|_2^2,
\]
while the right-hand side becomes
\[
g(x,x)\,h(0),
\]
which is independent of \(p\). Hence for every fixed \(x\),
\[
\left\|\sigma\bigl(\bR_{\rr}(p)x\bigr)\right\|_2^2
\]
is constant in \(p\).

\noindent\textbf{Step 2: Restrict to a single 2D rotary block.}
Since \(\bR_{\rr}(p)\) is block-diagonal, it suffices to consider a single \(2\times2\) rotary block with
non-zero frequency. Let \(t\) denote the rotation angle of that block. Choose \(x\) such that only this
block is non-zero, and within that block,
\[
x=
\begin{pmatrix}
a\\
0
\end{pmatrix},
\qquad a>0.
\]
Since \(\sigma(0)=0\) for \(\SiLU\), all other blocks contribute zero to the norm squared.

Thus we obtain
\[
\left\|\sigma\bigl(\bR_{\rr}(p)x\bigr)\right\|_2^2
=
\sigma(a\cos t)^2
+
\sigma(a\sin t)^2.
\]
By Step 1, this quantity must be independent of \(t\). Define
\[
F(t)
=
\sigma(a\cos t)^2
+
\sigma(a\sin t)^2.
\]
Then \(F(t)\) is constant in \(t\) for every \(a>0\).

\noindent\textbf{Step 3: Compute the second derivative at \(t=0\).}
Let
\[
G(u)=\sigma(u)^2.
\]
For \(u=a\cos t\) and \(v=a\sin t\), we have
\[
F(t)=G(u)+G(v).
\]
At \(t=0\),
\[
u(0)=a,\quad u'(0)=0,\quad u''(0)=-a,
\]
and
\[
v(0)=0,\quad v'(0)=a,\quad v''(0)=0.
\]
Since
\[
F''(t)
=
G''(u)(u')^2
+
G'(u)u''
+
G''(v)(v')^2
+
G'(v)v'',
\]
substituting \(t=0\) gives
\[
F''(0)
=
G''(a)\cdot 0
+
G'(a)\cdot(-a)
+
G''(0)\cdot a^2
+
G'(0)\cdot 0.
\]
Thus
\[
F''(0)
=
-aG'(a)
+
a^2G''(0).
\]

\noindent\textbf{Step 4: Use the derivatives of \(\SiLU\).}
For \(\SiLU\),
\[
\sigma(0)=0,\qquad \sigma'(0)=\frac12,\qquad \sigma''(0)=\frac12.
\]
Since \(G(u)=\sigma(u)^2\), we have
\[
G'(u)=2\sigma(u)\sigma'(u),
\]
and
\[
G''(u)
=
2\left(\sigma'(u)^2+\sigma(u)\sigma''(u)\right).
\]
In particular,
\[
G''(0)
=
2\left(\sigma'(0)^2+\sigma(0)\sigma''(0)\right)
=
2\left(\frac14+0\right)
=
\frac12.
\]
Therefore,
\[
F''(0)
=
-2a\sigma(a)\sigma'(a)
+
\frac{a^2}{2}.
\]

\noindent\textbf{Step 5: Contradiction from constancy of \(F(t)\).}
Since \(F(t)\) is constant in \(t\), we must have
\[
F''(0)=0.
\]
Hence
\[
-2a\sigma(a)\sigma'(a)
+
\frac{a^2}{2}
=
0.
\]
For \(a>0\), this implies
\[
\sigma(a)\sigma'(a)
=
\frac{a}{4}.
\]
Now observe that
\[
\frac{d}{da}\left(\sigma(a)^2\right)
=
2\sigma(a)\sigma'(a).
\]
Thus
\[
\frac{d}{da}\left(\sigma(a)^2\right)
=
\frac{a}{2}.
\]
Integrating from \(0\) to \(a\) and using \(\sigma(0)=0\), we obtain
\[
\sigma(a)^2
=
\int_0^a \frac{t}{2}\,dt
=
\frac{a^2}{4}.
\]
Therefore, for all \(a>0\),
\[
\sigma(a)^2
=
\frac{a^2}{4}.
\]
In particular, for \(a=1\),
\[
\sigma(1)^2
=
\frac14.
\]
However, for \(\SiLU\),
\[
\sigma(1)
=
\mathrm{SiLU}(1)
=
\frac{1}{1+e^{-1}},
\]
and
\[
\left(\frac{1}{1+e^{-1}}\right)^2
\neq
\frac14,
\]
because \(\frac{1}{1+e^{-1}}\neq \frac12\). This is a contradiction.

Therefore, the assumed factorization cannot exist. Hence an inside-activation rotation cannot yield a purely
relative-position low-rank kernel.
\end{proof}
\paragraph{What the obstruction does and does not say.}
\Cref{prop:inside} rules out an \emph{exact} relative factorization; it does not claim that inside
activation cannot be trained. A trained inside branch may approximate relative behavior, and the
proposition predicts that any such approximation must carry residual absolute-position dependence.
\Cref{tab:mechanism} confirms exactly this regime: the inside branch remains position-aware (offset
range $0.47$) yet has shift residual $3.7\times10^{-1}$, whereas the outside branch realizes the exact
relative structure (shift residual $1.8\times10^{-5}$).

\section{Additional Wan2.2 and VBench 2.0 Results}
\label{app:wan22}

Table~\ref{tab:wan22} extends the main comparison to Wan2.2-T2V, evaluated under the official
recommended setting with \(1280\times720\) resolution, 81 frames, and 40 denoising steps.
Table~\ref{tab:vbench2} further evaluates RoLA against the corresponding full-attention
baseline on Wan2.1-14B and Wan2.2-A14B-T2V using selected VBench 2.0~\citep{zheng2025vbench}
dimensions. We report Dynamic Spatial Relationship (DSR), Instance Preservation (IP),
Multi-View Consistency (MVC), Motion Rationality (MR), Motion Order Understanding (MOU),
and Complex Landscape (CL), which emphasize long-range relational, geometric, temporal,
and complex-scene consistency that can be affected by aggressive attention sparsification.

\begin{table}[h]
\centering
\small
\renewcommand{\arraystretch}{1.08}
\setlength{\tabcolsep}{4.5pt}
\caption{Additional quantitative comparison on Wan2.2-T2V benchmarks. The layout mirrors the main
quantitative table. All measurements use NVIDIA H100 80G GPUs, and latency refers to the end-to-end
runtime of the DiT denoising component (excluding text encoding and VAE stages). For the Wan2.2 MoE
architecture, this latency measures the denoising computation only and does not include the
high-noise/low-noise expert switching time.}
\label{tab:wan22}
\begin{tabular}{llccccccc}

\arrayrulecolor{tridentRed}
\toprule

\rowcolor{tridentHeader}
& & \multicolumn{5}{c}{\textbf{VBench Quality}} & \multicolumn{2}{c}{\textbf{Efficiency}} \\

\arrayrulecolor{tridentRed}
\hhline{~~-------}

\rowcolor{tridentHeader}
\multirow{-2}{*}{\textbf{Model}} & \multirow{-2}{*}{\textbf{Method}} & \textbf{SC}$\uparrow$ & \textbf{MS}$\uparrow$ & \textbf{BC}$\uparrow$ & \textbf{IQ}$\uparrow$ & \textbf{AQ}$\uparrow$ & \textbf{Latency}$\downarrow$ & \textbf{Sparsity}$\uparrow$ \\

\arrayrulecolor{tridentOrange}
\midrule

\multirow{5}{*}{Wan2.2-A14B-T2V} & Full & 0.9701 & 0.9866 & 0.9629 & 0.7102 & 0.6518 & 1584.6s & $0\%$ \\
& VSA & 0.9455 & 0.9754 & 0.9497 & 0.6522 & 0.5927 & 986.8s & $90\%$ \\
& SLA & 0.9572 & 0.9812 & 0.9578 & 0.6688 & 0.6182 & 883.9s & $90\%$\\
& RoPeSLR & 0.9621 & 0.9850 & 0.9533 & 0.6702 & 0.6193 & 881.4s & $90\%$\\
& \cellcolor{tridentLight}\textcolor{tridentRed}{\textbf{RoLA (Ours)}} & \cellcolor{tridentLight}0.9645 & \cellcolor{tridentLight}0.9856& \cellcolor{tridentLight}0.9589 & \cellcolor{tridentLight}0.6915 & \cellcolor{tridentLight}0.6401 & \cellcolor{tridentLight}\best{548s} & \cellcolor{tridentLight}$90\%$ \\

\arrayrulecolor{tridentRed}
\bottomrule

\end{tabular}
\arrayrulecolor{black}
\end{table}

\begin{table}[h]
\centering
\small
\renewcommand{\arraystretch}{1.08}
\setlength{\tabcolsep}{4.5pt}
\caption{Selected VBench 2.0 dimensions on Wan2.1-14B and Wan2.2-A14B-T2V.}
\label{tab:vbench2}
\begin{tabular}{llcccccc}

\arrayrulecolor{tridentRed}
\toprule

\rowcolor{tridentHeader}
\textbf{Model} & \textbf{Method} & \textbf{DSR}$\uparrow$ & \textbf{IP}$\uparrow$ & \textbf{MVC}$\uparrow$ & \textbf{MR}$\uparrow$ & \textbf{MOU}$\uparrow$ & \textbf{CL}$\uparrow$ \\

\arrayrulecolor{tridentOrange}
\midrule

\multirow{2}{*}{Wan2.1-14B} & Full & 30.16 & 83.67 & 24.09 & 28.67 & 23.34 & 15.98 \\
& \cellcolor{tridentLight}\textcolor{tridentRed}{\textbf{RoLA (Ours)}} & \cellcolor{tridentLight}29.73 & \cellcolor{tridentLight}80.05 & \cellcolor{tridentLight}22.56 & \cellcolor{tridentLight}26.22 & \cellcolor{tridentLight}21.69 & \cellcolor{tridentLight}15.08 \\

\arrayrulecolor{tridentOrange}
\midrule

\multirow{2}{*}{Wan2.2-A14B-T2V} & Full & 41.03 & 86.82 & 32.63 & 34.54 & 28.56 & 19.02 \\
& \cellcolor{tridentLight}\textcolor{tridentRed}{\textbf{RoLA (Ours)}} & \cellcolor{tridentLight}39.26 & \cellcolor{tridentLight}83.88 & \cellcolor{tridentLight}29.94 & \cellcolor{tridentLight}31.82 & \cellcolor{tridentLight}26.12 & \cellcolor{tridentLight}18.55 \\

\arrayrulecolor{tridentRed}
\bottomrule

\end{tabular}
\arrayrulecolor{black}
\end{table}

\section{Additional Ablation Results}
\label{app:additional_ablations}

The following experiments are designed to address two potential confounds in the main comparison. First,
we fix the sparse attention engine, routing pattern, and target sparsity while changing only the global
branch. Second, we evaluate RoLA at multiple sparsity levels. Both tables report the five VBench
dimensions used in the main experiment.

\begin{table}[t]
\centering
\small
\renewcommand{\arraystretch}{1.08}
\setlength{\tabcolsep}{5.0pt}
\caption{Full-attention fine-tuning control on Wan2.1-1.3B at 480p. The first row copies the pretrained
Full Attention result from Table~\ref{tab:main}; the second row keeps dense attention and applies the same
single-stage fine-tuning protocol, adaptation data, batch size, and optimization steps as RoLA.}
\label{tab:abl_full_ft}
\begin{tabular}{lccccc}
\arrayrulecolor{tridentRed}
\toprule
\rowcolor{tridentHeader}
\textbf{Setting} & \textbf{SC}$\uparrow$ & \textbf{MS}$\uparrow$ & \textbf{BC}$\uparrow$
& \textbf{IQ}$\uparrow$ & \textbf{AQ}$\uparrow$ \\
\arrayrulecolor{tridentOrange}
\midrule
Full Attention & 0.9403 & 0.9557 & 0.9498 & 0.6510 & 0.6088 \\
Full Attention + matched fine-tuning & 0.9409 & 0.9556 & 0.9495 & 0.6508 & 0.6084 \\
\arrayrulecolor{tridentRed}
\bottomrule
\end{tabular}
\arrayrulecolor{black}
\end{table}

\begin{table}[t]
\centering
\small
\renewcommand{\arraystretch}{1.08}
\setlength{\tabcolsep}{5.0pt}
\caption{Controlled global-branch ablation on Wan2.1-1.3B at 90\% sparsity. The sparse engine and routing
pattern are fixed across rows; only the global branch is changed.}
\label{tab:abl_global}
\begin{tabular}{lccccc}
\arrayrulecolor{tridentRed}
\toprule
\rowcolor{tridentHeader}
\textbf{Global branch} & \textbf{SC}$\uparrow$ & \textbf{MS}$\uparrow$ & \textbf{BC}$\uparrow$
& \textbf{IQ}$\uparrow$ & \textbf{AQ}$\uparrow$ \\
\arrayrulecolor{tridentOrange}
\midrule
None (sparse only) & 0.8480 & 0.8912 & 0.9530 & 0.4939 & 0.5281 \\
No RoPE & 0.9301 & 0.9253 & 0.9641 & 0.5817 & 0.6082 \\
RoPE inside activation & 0.9303 & 0.9257 & 0.9639 & 0.6027 & 0.6185 \\
\rowcolor{tridentLight}
\textcolor{tridentRed}{\textbf{RoPE outside (ours)}} & 0.9521 & 0.9812 & 0.9490 & 0.6492 & 0.6068 \\
\arrayrulecolor{tridentRed}
\bottomrule
\end{tabular}
\arrayrulecolor{black}
\end{table}

\begin{table}[t]
\centering
\small
\renewcommand{\arraystretch}{1.08}
\setlength{\tabcolsep}{5.0pt}
\caption{Sparsity sweep for RoLA on Wan2.1-1.3B at 480p.}
\label{tab:abl_sparsity}
\begin{tabular}{lccccc}
\arrayrulecolor{tridentRed}
\toprule
\rowcolor{tridentHeader}
\textbf{Sparsity} & \textbf{SC}$\uparrow$ & \textbf{MS}$\uparrow$ & \textbf{BC}$\uparrow$
& \textbf{IQ}$\uparrow$ & \textbf{AQ}$\uparrow$ \\
\arrayrulecolor{tridentOrange}
\midrule
80\% & 0.9528 & 0.9820 & 0.9495 & 0.6500 & 0.6072 \\
\rowcolor{tridentLight}
\textcolor{tridentRed}{\textbf{90\% (default)}} & 0.9521 & 0.9812 & 0.9490 & 0.6492 & 0.6068 \\
95\% & 0.9503 & 0.9785 & 0.9401 & 0.6485 & 0.5929 \\
\arrayrulecolor{tridentRed}
\bottomrule
\end{tabular}
\arrayrulecolor{black}
\end{table}

\section{Training and Model Hyperparameters}
\label{app:hyper}

\subsection{Why Single-Stage Fine-Tuning Replaces the Prior Two-Stage Recipe}
\label{app:onestage}

RoPeSLR~\citep{liu2026ropeslr} used a two-stage post-training procedure because its low-rank Fourier
compensator is not a genuine linear-attention branch: it first learns a coordinate-MLP surrogate together with an absolute
positional module, and only then co-adapts that surrogate with the full backbone. RoLA changes this
optimization picture. Our added parameters
$$
\Theta_{\mathrm{new}}=\{\Pq,\Pk,W_g,b_g\}
$$
sit directly on the final low-rank attention path, and their outputs are fused with the sparse branch through
the same forward computation used at inference.

This architecture makes a single joint stage the natural default for three reasons. First, the branch is
initialized near the sparse baseline because the gate starts almost closed ($W_g=0$, $b_g<0$), so early
optimization does not abruptly perturb the pre-trained model. Second, the projections $\Pq,\Pk$, the gate,
and the backbone attention statistics must co-adapt under the final diffusion objective; training them
separately would introduce an avoidable stage mismatch between branch fitting and whole-model behavior.
Third, unlike the absolute-PE variant of the low-rank Fourier compensator~\citep{liu2026ropeslr}, RoLA does
not add a separate positional module whose standalone stabilization motivates a warm-up stage.

We therefore adopt a single-stage full fine-tuning recipe for RoLA. This choice should be read as an
architecture-matched training strategy rather than a blanket claim that two-stage optimization is inferior for
all sparse--low-rank designs. To keep this distinction explicit, we report a controlled single-stage vs.\
two-stage ablation in \cref{app:trainabl}; Table~\ref{tab:abl_train} reports the corresponding
comparison under the same VBench setting.

\subsection{Training-Strategy Ablation}
\label{app:trainabl}

Table~\ref{tab:abl_train} isolates the optimization question from the architectural one. This
comparison keeps the backbone, sparsity, data, optimizer, and total update budget fixed, and changes only the
post-training schedule: \emph{single-stage} means the joint optimization used by RoLA, whereas
\emph{two-stage} follows the prior RoPeSLR-style warm-up-then-joint recipe. We evaluate both strategies
under the same VBench setting to compare their downstream generation quality.

\begin{table}[h]
\centering
\small
\renewcommand{\arraystretch}{1.10}
\setlength{\tabcolsep}{4.5pt}
\caption{Controlled ablation of post-training strategy.}
\label{tab:abl_train}
\begin{tabular}{lccccc}
\arrayrulecolor{tridentRed}
\toprule
\rowcolor{tridentHeader}
\textbf{Strategy} & \textbf{SC}$\uparrow$ & \textbf{MS}$\uparrow$ & \textbf{BC}$\uparrow$ & \textbf{IQ}$\uparrow$ & \textbf{AQ}$\uparrow$ \\
\arrayrulecolor{tridentOrange}
\midrule
\rowcolor{tridentLight}
\textcolor{tridentRed}{\textbf{Single-stage (ours)}} & 0.9629 & 0.9853 & 0.9537 & 0.6783 & 0.6294 \\
Two-stage (RoPeSLR-style) & 0.9631 & 0.9842 & 0.9521 & 0.6778 & 0.6296 \\
\arrayrulecolor{tridentRed}
\bottomrule
\end{tabular}
\arrayrulecolor{black}
\end{table}

Table~\ref{tab:hyper} lists the main hyperparameters used for the single-stage full fine-tuning
(\cref{sec:posttrain}). The newly introduced parameters
$\Theta_{\mathrm{new}}=\{\Pq,\Pk,W_g,b_g\}$ use a higher learning rate than the backbone, since
they are trained from initialization while the backbone is only lightly adapted. This stage is
deliberately lightweight: RoLA adapts an already pre-trained video DiT rather than learning
video generation from scratch. The rotary pairing and frequency schedule are directly inherited
from the pre-trained backbone and fixed by the architecture, so the new projections only need to
adapt the pre-trained query/key features to the low-rank feature space. In addition, the gate is
initialized to a small value, allowing the new branch to be introduced gradually. In our setting,
2,000 training samples over four epochs are sufficient for effective adaptation.

\begin{table}[h]
\centering
\small
\renewcommand{\arraystretch}{1.10}
\setlength{\tabcolsep}{7.0pt}
\caption{Training and model hyperparameters for RoLA on Wan2.1-T2V-14B.}
\label{tab:hyper}
\begin{tabular}{ll}
\arrayrulecolor{tridentRed}
\toprule
\rowcolor{tridentHeader}
\textbf{Hyperparameter} & \textbf{Value} \\
\arrayrulecolor{tridentOrange}
\midrule
Optimizer & AdamW \\
$\beta_1,\beta_2$ & $0.9,\ 0.999$ \\
Weight decay & $10^{-5}$ \\
Precision & BF16 \\
Low-rank dimension $\rr$ & $64$ \\
Gate bias init $b_0$ & $-1.946$ \\
Global batch size & 32 \\
\arrayrulecolor{tridentOrange}
\midrule
\rowcolor{tridentLight}
\multicolumn{2}{l}{\textcolor{tridentRed}{\emph{Single-stage full fine-tuning}}} \\
Objective & Flow-matching \\
LR schedule & Cosine decay \\
Training samples & 2000 \\
Training epochs & 4 \\
\arrayrulecolor{tridentOrange}
\midrule
\rowcolor{tridentLight}
\multicolumn{2}{l}{\textcolor{tridentRed}{\emph{Learning rates}}} \\
Backbone & $2\times10^{-6}$ \\
Low-rank projections ($\Pq,\Pk$) & $5\times10^{-5}$ \\
Gate projection $W_g$ & $3\times10^{-5}$ \\
Gate bias $b_g$ & $5\times10^{-5}$ \\
\arrayrulecolor{tridentRed}
\bottomrule
\end{tabular}
\arrayrulecolor{black}
\end{table}

\section{Computational Complexity Analysis}
\label{app:complexity}

The rotary-positioned low-rank branch retains the $\bigo(L)$ scaling of linear attention: the global context
$$
C=\klr^\top V\in\R^{\rr\times \dd}
$$
is formed once at cost $\bigo(L\rr \dd)$ and reused across all queries, and the rank-truncated rotation adds
only $\bigo(L\rr)$ elementwise operations. Compared with the absolute-PE variant of the low-rank Fourier
compensator in \citet{liu2026ropeslr}, which also has $\bigo(L\rr \dd)$ complexity but carries extra learnable positional
parameters, RoLA introduces zero additional positional parameters by reusing the backbone rotary tensors.
For a fixed sparsity ratio and low-rank width, the asymptotic inference cost of the full module is dominated
by the sparse branch, while the low-rank branch contributes only a linear overhead term.

We give a detailed FLOPs analysis of the RoLA attention module during the forward pass, establishing two
claims: (i) the rotary low-rank branch retains linear complexity in sequence length, with its cost controlled
by the bottleneck width; and (ii) relative to pure sparse attention, the overhead of the low-rank branch is
asymptotically negligible for long sequences under a fixed sparsity ratio and low-rank width.

\noindent\textbf{Notation.}
Let $B$ be the batch size, $L$ the sequence length, $H$ the number of heads, $\dd$ the per-head dimension
($d_{\mathrm{model}}=H\dd$), $S$ the sparsity ratio (active fraction $1-S$), and $\rr$ the low-rank
bottleneck. We report dominant matmul FLOPs throughout. A matmul $\R^{M\times N}\cdot\R^{N\times P}$ costs
$2MNP$ FLOPs; lower-order operations, including Softmax, RoPE, RMSNorm, sigmoid, and gated addition, are
omitted from the FLOPs count.

\noindent\textbf{Cost anatomy.}
The module decomposes into three parts.

\emph{(1) Sparse branch.}
With block sparsity $S$, each query attends to $L(1-S)$ keys; the $QK^\top$ scores and
$\mathrm{Softmax}(\cdot)V$ aggregation give
\begin{equation}
C_{\mathrm{sparse}} = 4\,B H L^2 (1-S)\,\dd .
\end{equation}

\emph{(2) Low-rank branch.}
Down-projection $Q^{\mathrm{n}}\Pq^\top$ (and the key side) costs $2BL\dd\rr$ each; the context
$C=\klr^\top V$ costs $2BL\rr\dd$; the readout $\qlr C$ costs $2BL\rr\dd$. Summing over $H$ heads,
\begin{equation}
C_{\mathrm{lowrank}} = 8\,B H L \dd\,\rr .
\end{equation}

\emph{(3) Gating and fusion.}
The scalar-gate projection $X W_g$ costs $2BHL\dd$, i.e.,
\begin{equation}
C_{\mathrm{fusion}}=2BHL\dd .
\end{equation}
The total is therefore
\begin{equation}
C_{\mathrm{RoLA}} = 4BHL^2(1-S)\dd + 8BHL\dd\,\rr + 2BHL\dd .
\end{equation}

\noindent\textbf{Cost comparison with sparse-linear frameworks.}
A standard full-width linear branch with feature map
$\phi:\R^{\dd}\to\R^{\dd}$ forms
$K\!V_{\mathrm{ctx}}=\phi(K)^\top V$ ($2BHL\dd^2$)
and retrieves $\phi(Q)K\!V_{\mathrm{ctx}}$ ($2BHL\dd^2$),
giving a dominant matmul cost
$C_{\mathrm{linear}}=4BHL\dd^2$.
Under the same FLOPs accounting, our low-rank branch has
$C_{\mathrm{lowrank}}=8BHL\dd\,\rr$, including the query/key
down-projections, context aggregation, and readout. Therefore,
\begin{equation}
\frac{C_{\mathrm{lowrank}}}{C_{\mathrm{linear}}}
=
\frac{8BHL\dd\,\rr}{4BHL\dd^2}
=
\frac{2\rr}{\dd}.
\end{equation}
Thus, the low-rank branch is cheaper than a full-width linear branch when
$\rr<\dd/2$, has the same dominant matmul FLOPs when $\rr=\dd/2$, and remains
linear in sequence length in either case. Under our default setting
$\rr=64$ and $\dd=128$, the two branches have equal leading-order matmul
FLOPs; the advantage of RoLA therefore lies not in reducing this
leading-order cost, but in preserving relative-position structure
within a compressed rotary feature space.

\noindent\textbf{Asymptotically negligible overhead over pure sparse attention.}
The extra counted FLOPs relative to the pure sparse baseline are
\begin{equation}
\eta
=
\frac{C_{\mathrm{lowrank}}+C_{\mathrm{fusion}}}{C_{\mathrm{sparse}}}
=
\frac{8BHL\dd\,\rr + 2BHL\dd}
     {4BHL^2(1-S)\dd}
=
\frac{2\rr+0.5}{L(1-S)} .
\end{equation}
For the fixed sparsity ratio $S<1$ and fixed low-rank width $\rr$ used in our experiments,
the number of active keys per query, $L(1-S)$, grows linearly with the sequence length $L$.
Therefore,
\begin{equation}
\eta=\mathcal{O}(L^{-1}),
\qquad
\lim_{L\to\infty}\eta=0.
\end{equation}
Thus, under the fixed-sparsity regime considered in this work, the additional FLOPs introduced by
the low-rank branch and gated fusion become asymptotically negligible relative to the sparse-attention
computation as the video sequence length increases. In other words, RoLA adds only a linear-cost
global branch on top of the sparse attention module, while the dominant sparse-attention cost grows
quadratically with $L$ for a fixed sparsity ratio.

\raggedbottom

\section{Additional Qualitative Comparisons}
\label{app:qual_compare}
\vspace{-8pt} 
The following figure shows representative video samples. RoLA preserves better temporal coherence and scene layout under high sparsity compared with competing sparse-global methods.
\vspace{-6pt} 

\begin{figure}[H]
\centering
\setlength{\abovecaptionskip}{3pt}
\setlength{\belowcaptionskip}{0pt}
\setlength{\intextsep}{6pt}

\begin{subfigure}{0.86\textwidth}
\centering
\includegraphics[width=\linewidth]{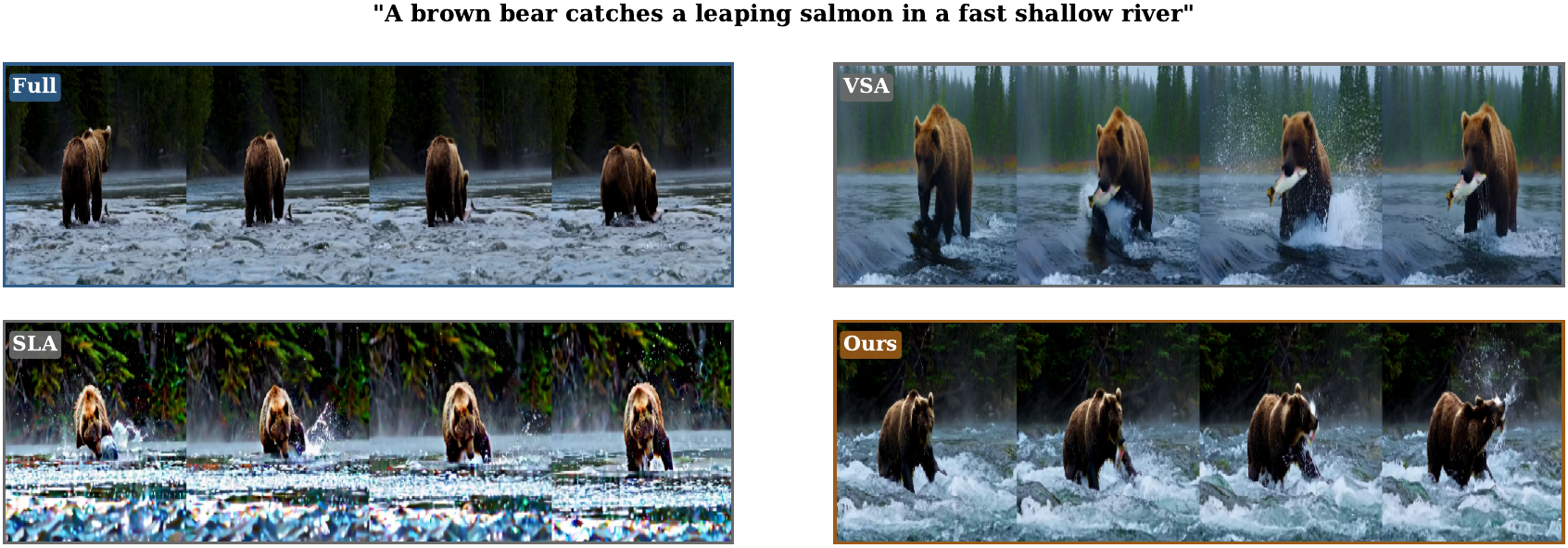}
\caption{}
\end{subfigure}

\vspace{1pt}

\begin{subfigure}{0.86\textwidth}
\centering
\includegraphics[width=\linewidth]{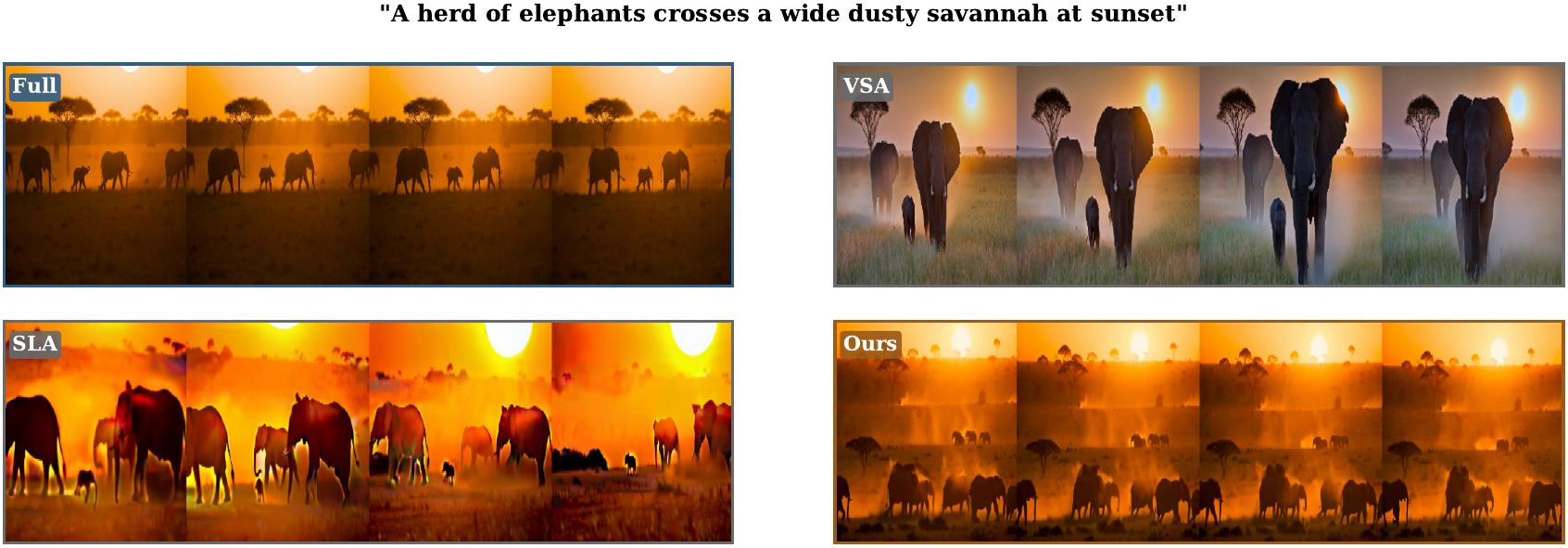}
\caption{}
\end{subfigure}

\vspace{1pt}

\begin{subfigure}{0.86\textwidth}
\centering
\includegraphics[width=\linewidth]{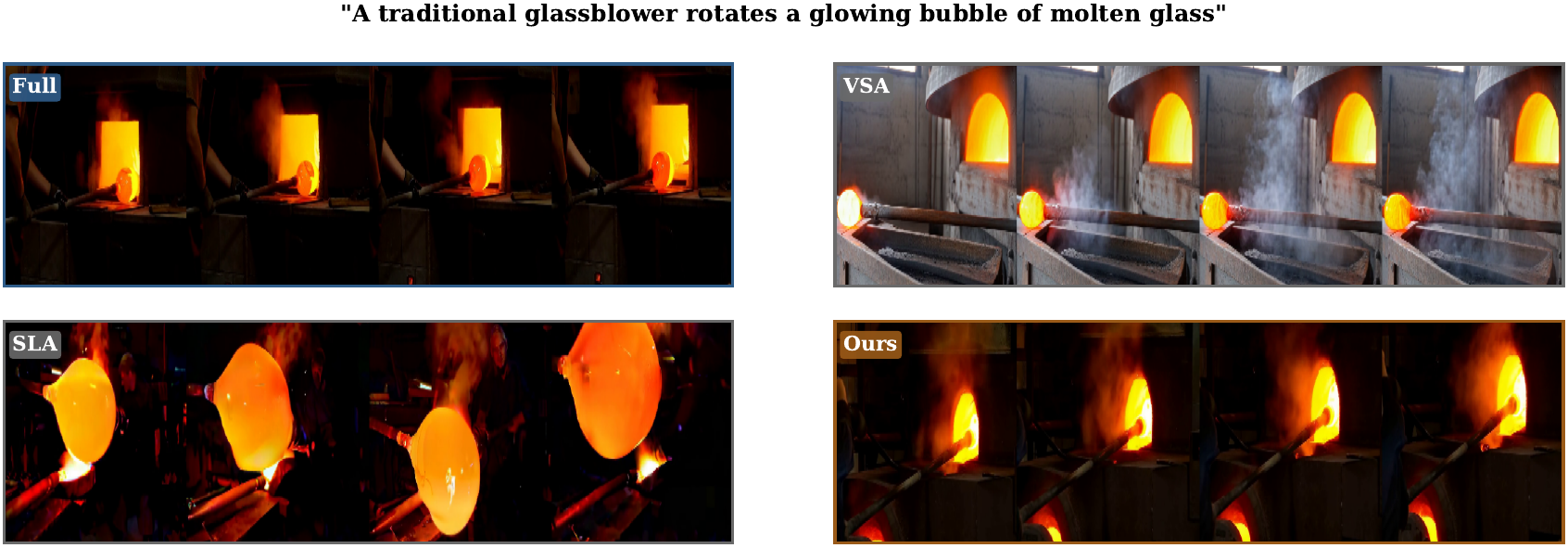}
\caption{}
\end{subfigure}

\caption{Qualitative comparison on five video prompts. (a) Bear, (b) Elephants, (c) Glassblower, (d) Lighthouse, (e) Tuscany.}
\label{fig:qualitative_compare}
\end{figure}

\begin{figure}[H]
\ContinuedFloat
\centering
\setcounter{subfigure}{3}
\setlength{\abovecaptionskip}{3pt}
\setlength{\belowcaptionskip}{0pt}
\setlength{\intextsep}{6pt}

\begin{subfigure}{0.86\textwidth}
\centering
\includegraphics[width=\linewidth]{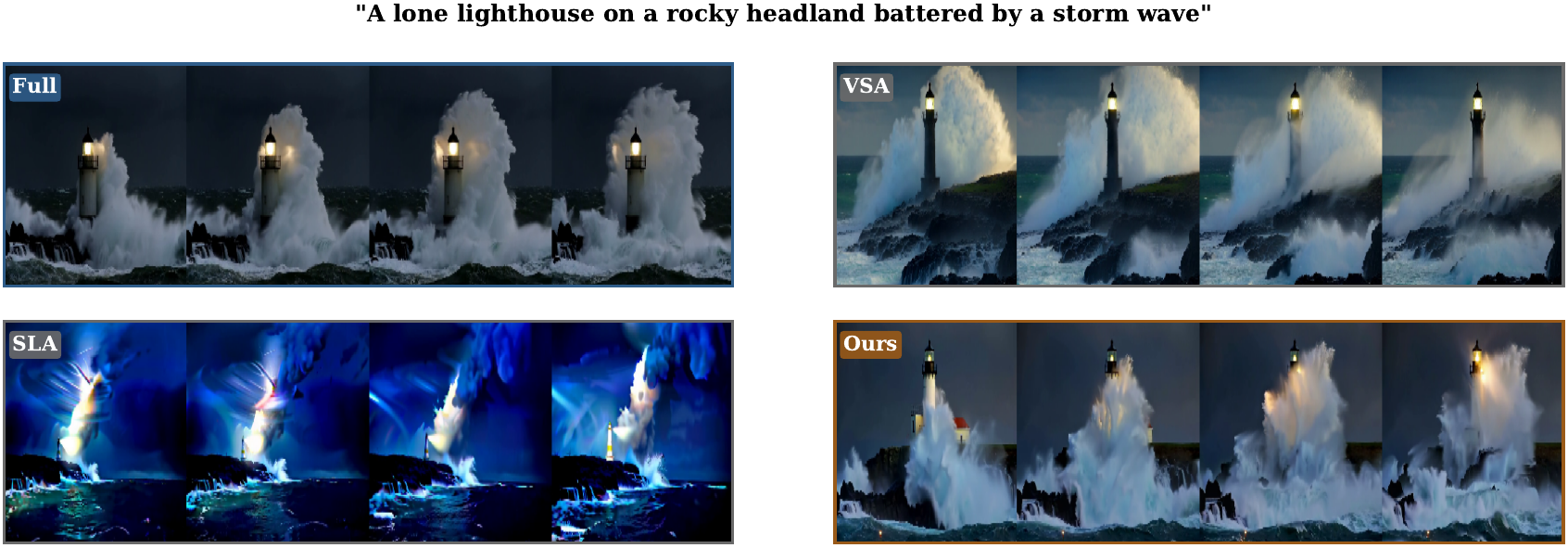}
\caption{}
\end{subfigure}

\vspace{1pt}

\begin{subfigure}{0.86\textwidth}
\centering
\includegraphics[width=\linewidth]{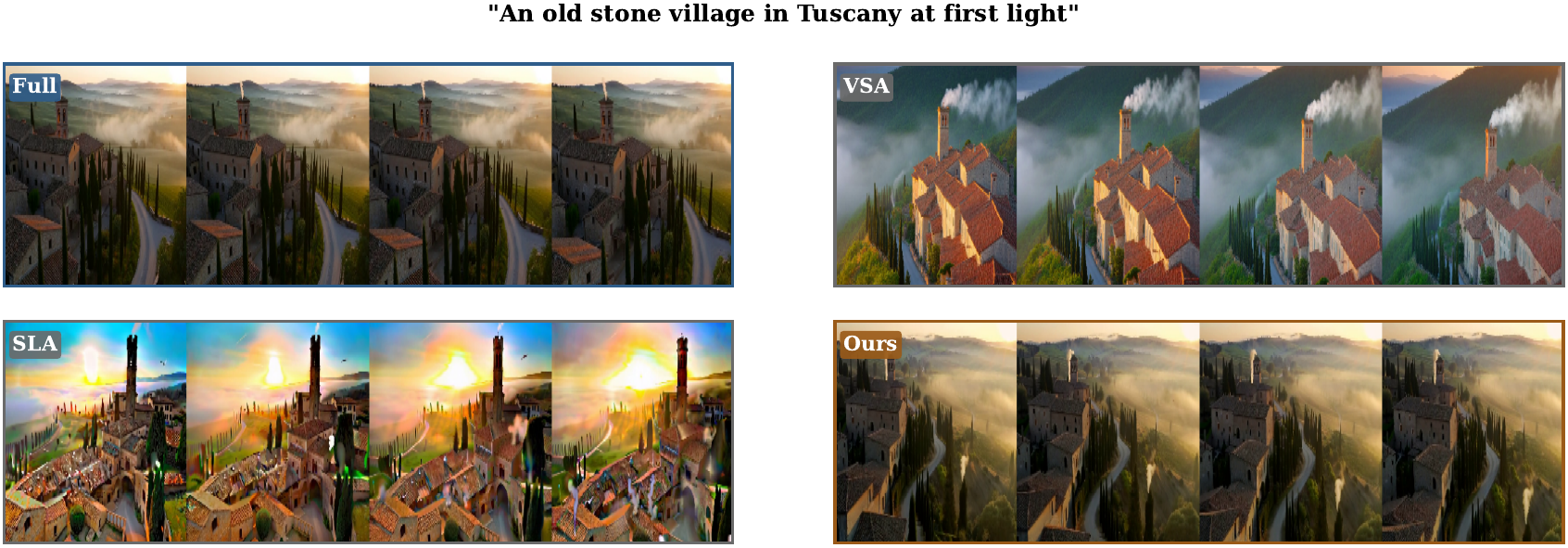}
\caption{}
\end{subfigure}

\caption{Qualitative comparison on five video prompts (continued). (d) Lighthouse, (e) Tuscany.}
\end{figure}

\clearpage
\flushbottom

\end{document}